\documentclass{article}
\usepackage[margin=1.2in]{geometry}
\usepackage{latexsym, amsthm, amscd, amsfonts, mathrsfs, amsmath, amssymb, stmaryrd, tikz-cd, mathrsfs, bbm, url, esint, mathtools, enumerate}
\usepackage{blkarray}
\usepackage{dsfont}
\usepackage[affil-it]{authblk}

\usepackage[OT4]{fontenc}
\usepackage[margin=1.2in]{geometry}
\usepackage{amssymb,amsfonts,amsmath,amsthm,amscd,dsfont,mathrsfs,pifont}
\usepackage{blkarray}
\usepackage{graphicx,float,psfrag,epsfig,color}
\usepackage{qtree}
\usepackage{comment}
\usepackage{authblk,textcomp}
\usepackage{caption} 
\usepackage{subcaption}

\usepackage{latexsym, amsthm, amscd, amsfonts, mathrsfs, amsmath, amssymb, stmaryrd, tikz-cd, mathrsfs, bbm, url, esint, mathtools, enumerate, caption, subcaption}
 \usepackage{dsfont}
\usepackage{blkarray}
\usepackage[utf8]{inputenc} 
\usepackage[T1]{fontenc}    
\usepackage{hyperref}       
\usepackage{url}            
\usepackage{booktabs}       
\usepackage{nicefrac}       
\usepackage{microtype}      

\usepackage{algpseudocode}
\usepackage{enumitem}
\usepackage{comment}
\usepackage{bbm}
\usepackage[ruled,vlined]{algorithm2e}

\newcommand{\pr}{\mathbb{P}}
\newcommand{\E}{\mathbb{E}}

\DeclareMathOperator{\Var}{Var}

\DeclareMathOperator{\argmin}{argmin}

\DeclareMathOperator{\ReLU}{ReLU}
\DeclareMathOperator{\Unif}{Unif}
\DeclareMathOperator{\Rad}{Rad}

\newcommand{\NN}{N}
\DeclareMathOperator{\poly}{poly}

\DeclareMathOperator{\Inf}{Inf}

\newcommand{\cX}{\mathcal{X}}

\newcommand{\cD}{\mathcal{D}}
\newcommand{\cN}{\mathcal{N}}

\newcommand{\bR}{\mathbb{R}}

\newcommand{\bN}{\mathbb{N}}

\newcommand{\cL}{\mathcal{L}}
\newcommand{\cB}{\mathcal{B}}
\newcommand{\bS}{\mathbb{S}}

\newcommand{\bmu}{\boldsymbol{\mu}}

\definecolor{mydarkblue}{rgb}{0,0.08,0.45}
  \hypersetup{ %
    colorlinks=true,
    linkcolor=mydarkblue,
    citecolor=mydarkblue,
    filecolor=mydarkblue,
    urlcolor=mydarkblue,
    }
\definecolor{DSgray}{cmyk}{0,0,0,0.7}
\definecolor{DSred}{cmyk}{0,0.7,0,0.7}

\theoremstyle{plain}
\newtheorem{definition}{Definition}
\theoremstyle{plain}
\newtheorem{theorem}{Theorem}
\theoremstyle{plain}
\newtheorem{proposition}[theorem]{Proposition}
\theoremstyle{plain}
\theoremstyle{plain}
\theoremstyle{plain}
\theoremstyle{plain}
\theoremstyle{plain}
\newtheorem{lemma}{Lemma}
\theoremstyle{plain}
\newtheorem{corollary}{Corollary}
\theoremstyle{plain}
\theoremstyle{plain}
\theoremstyle{plain}
\newtheorem{assumption}{Assumption}
\theoremstyle{plain}
\theoremstyle{plain}

\title{
Low-dimensional functions are efficiently learnable\\ under randomly biased distributions}

\author[1,2]{
Elisabetta Cornacchia}
\author[1,3]{Dan Mikulincer}
\author[1]{Elchanan Mossel}

\affil[1]{\small Massachusetts Institute of Technology (MIT).}
\affil[2]{\small INRIA, DI/ENS, PSL.}
\affil[3]{\small University of Washington (UW).}

\date{}

\usepackage[parfill]{parskip}

\begin{document}

\maketitle

\begingroup
    \renewcommand\thefootnote{}  
\footnotetext{\small 
\noindent
Authors are listed in alphabetical order. Email: \texttt{elisabetta.cornacchia@inria.fr}, \texttt{danmiku@uw.edu}, \texttt{elmos@mit.edu}.}
\endgroup

\begin{abstract}%
The problem of learning single-index and multi-index models has gained significant interest as a fundamental task in high-dimensional statistics. Many recent works have analyzed gradient-based methods, particularly in the setting of isotropic data distributions, often in the context of neural network training.
Such studies have uncovered precise characterizations of algorithmic sample complexity in terms of certain analytic properties of the target function, such as the leap, information, and generative exponents. These properties establish a quantitative separation between low- and high-complexity learning tasks. In this work, we show that high-complexity cases are rare. Specifically, we prove that introducing a small random perturbation to the data distribution—via a random shift in the first moment—renders any Gaussian single-index model as easy to learn as a linear function. We further extend this result to a class of multi-index models, namely sparse Boolean functions, also known as Juntas.
\end{abstract}

\section{Introduction}
The demonstrated successes of modern deep learning paradigms can be attributed, at least in part, to the following conjectured phenomenon: while some natural learning tasks are inherently high-dimensional, often involving millions of features, they exhibit a latent low-dimensional structure (\cite{goldt2020modeling}). This structure reduces the statistical and computational complexities of learning algorithms, transforming seemingly intractable problems into solvable instances. In this work, we aim to further elucidate the complexity of learning problems with low-dimensional structure, showing that, under mild conditions, the presence of noise can make such problems tractable, even when they would otherwise remain challenging.

Specifically, we shall focus on the \emph{single-index model}, a classical family of models that capture low-dimensional structure in target data. In this model, there exists a target function $f:\bR \to \bR$, which may be known or unknown, and an unknown signal $w^*\in \bS^d$. The goal is to estimate $(f,w^*)$ from samples $(x_i,y_i)_{i=1}^n$ where for every $i \in [n]$, $x_i \in \bR^d$, $y_i = f(\langle w^*,x_i\rangle) + \zeta_i$, with $\zeta_i$ being \emph{i.i.d.} sub-Gaussian centered random variables. Thus, while this problem `lives' in $d$-dimensions, it fundamentally depends on a single direction, revealing an underlying one-dimensional structure.

Owing to the rotational invariance of the single-index model, it is standard to assume that $(x_i)_{i=1}^n$ are \emph{i.i.d.} samples from the standard Gaussian distribution. Under this assumption, the complexity of many algorithms is well-understood, with numerous results available in the literature. Among these, a prominent contribution is due to~\cite{arous2021online} (see also an earlier work by~\cite{dudeja2018learning}) who investigated the performance of spherical\footnote{Since the signal $w^*$ is constrained to lie on the sphere, it is natural to restrict the dynamics of SGD to the sphere as well.} online stochastic gradient descent (SGD) for recovering $w^*$ when $f$ is known. They showed that under suitable integrability conditions on \( f \) and its derivative, the performance of online SGD is determined by a single parameter: The so-called \emph{information-exponent} or \emph{Hermite rank} of \( f \), which we denote by \( \mathcal{I}(f) \). Roughly,
\( \mathcal{I}(f) \) is the smallest non-zero integer \( k \) for which the \( k \)-th Hermite coefficient of \( f \) is non-zero (see Section~\ref{sec:gaussian_results} for a formal definition of \( \mathcal{I}(f) \)).
 Several subsequent works showed the importance of the information-exponent for stochastic optimization problems in various settings~(see for example~\cite{ba2022high,bietti2022learning,bietti2023learning} and the related work section). A key takeaway from these results is the following: if $\mathcal{I}(f)$ is large, then the single-index model can be difficult to learn. Beyond the online setting, \cite{damian2024computational,dandi2024benefits,lee2024neural,arnaboldi2024repetita} showed that if samples are reused, SGD can overcome the limitations imposed by the information-exponent. In this setting, its complexity is governed by the \emph{generative-exponent}, defined as the smallest information-exponent attained by any $L_2$ transformation of $f$.

In this work we show that these conditions are remarkably fragile. By slightly perturbing the function $f$ with a random shift, we can ensure that $\mathcal{I}(f)$ becomes small and that the first Hermite coefficient of $f$ remains of constant order. In particular, such a slight perturbation can be achieved by shifting the first moment of the input distribution. Because, by definition, the generative-exponent is never larger than the information-exponent, our results imply that the generative-exponent also becomes small. This leads to the following general principle:

\begin{center}
\emph{By introducing a minimal amount of randomness to the Gaussian single-index model,\\ any target function becomes efficiently learnable.}
\end{center}
To clarify, the \emph{minimal randomness} refers to the introduction of random perturbations in the form of a shift, and \emph{efficiently learnable} means that the complexity becomes independent of the specific target function.

\subsection{Summary of contributions}

\paragraph{Gaussian single-index models.} We study the problem of learning single-index models under randomly shifted Gaussian inputs. Specifically, we assume that the input follows $x \sim \mathcal{N}(\alpha, \mathrm{I}_d)$, where the shift $\alpha$ is drawn from $\mathcal{N}(0, \mathrm{I}_d)$. Our main result shows that, under appropriate assumptions, any such target function is as easy to learn as a linear function, independently of the specific target function and its information-exponent. The key technical contribution is proving that, under a random shift, the first Hermite coefficient of any target function remains of constant order (Theorem~\ref{thm:infromalsmallball}). Under appropriate conditions, we establish the efficient learnability both in the \textit{parametric setting}, where the link function $f$ is known (Theorem~\ref{thm:parameteric}), and in the \textit{semi-parametric setting}, where $f$ is unknown and is learned by a shallow $\ReLU$ neural network (Theorem~\ref{thm:semiparameteric}).

\paragraph{Sparse Boolean functions.} Furthermore, we study a class of multi-index models known as sparse Boolean functions, or Juntas. In these tasks, the input is a Boolean vector $x \in \{ \pm 1 \}^d$, and the target function $f:\{ \pm 1 \}^d \to \mathbb{R}$ depends only on a small subset of the input coordinates. Prior work has focused on learning such targets under a uniform input distribution, introducing measures based on the Fourier-Walsh expansion of $f$ (\cite{abbe2023sgd,joshi2024complexity}; see the related work section for details). We focus on the case where the input is randomly shifted and show that stochastic gradient descent (SGD) on a two-layer $\ReLU$ network can learn any sparse Boolean function with linear sample complexity, regardless of the specific target (Theorem~\ref{thm:juntas_formal}). Similar to the Gaussian case, we prove that a random shift in the input distribution ensures that all first-order Fourier-Walsh coefficients on the relevant coordinates remain of constant order (Proposition~\ref{prop:fourier_firstorder}).


\subsection{Related work}

\paragraph{Single-index and multi-index models on Gaussian isotropic data.} Several studies have investigated the learning dynamics of single and multi-index models with isotropic Gaussian data, aiming to understand how shallow neural networks adapt to low-dimensional structures. Some works have analyzed the effect of a single step of gradient descent with a large learning rate~(\cite{ba2022high,damian2022neural,dandi2023two}), establishing separations between neural networks, random feature models, and kernel methods. For multiple steps of online (i.e., one-pass) stochastic gradient descent (SGD),~\cite{arous2021online} showed that the sample complexity for learning a single-index model is $\Theta\left(d^{\max(\mathcal{I}(f)-1,1)}\right)$ (up to logarithmic factors), where $\mathcal{I}(f)$ is the information-exponent, assuming initialization from a uniform distribution on the sphere. More recently,~\cite{damian2024smoothing} demonstrated that smoothing the loss landscape can improve this complexity to $\Theta\Big(d^{\max(\frac{\mathcal{I}(f)}{2},1)}\Big)$. Beyond this setting,~\cite{bietti2022learning} studied the semi-parametric learning of target functions using shallow $\ReLU$ networks. Other works explored variants of SGD~(\cite{berthier2024learning,ben2022high,chen2023learning}). Recent advances have shown that reusing data batches can overcome the limitations imposed by the information-exponent. This has led to new complexity measures, based on the generative-exponent~(\cite{damian2024computational,dandi2024benefits}) and on the Approximate Message Passing (AMP) algorithm~(\cite{troiani2024fundamental}). Notably, for single-index models with polynomial target functions, \cite{arnaboldi2024repetita,lee2024neural} showed that learning can be achieved by multi-pass SGD with sample complexity scaling as $\Theta(d \log(d))$. In contrast, our work applies to all Lipschitz targets, and we show linear sample complexity independently of the information and generative-exponents.








\paragraph{Beyond Gaussian isotropic data.}
Several studies have established lower bounds for gradient-based learning under specific choices of target functions and data distributions, highlighting the crucial role both elements play in determining learning success~(\cite{yehudai2020learning,goel2020superpolynomial,shamir2018distribution}). Conversely, positive results have been demonstrated for certain target functions under mild distributional assumptions~(\cite{frei2020agnostic,wu2022learning,song2021cryptographic}).
For more general target functions, some works have moved beyond the standard isotropic data assumptions and considered Gaussian distributions with spiked-covariance structure. These studies show that additional structure in the data can significantly improve learning efficiency compared to isotropic settings~(\cite{ba2024learning,mousavi2023gradient,nitanda2024improved}). Notably,~\cite{mousavi2023gradient} demonstrated that when the spike is sufficiently large, the learning complexity is $O(d^{3+\varepsilon})$ (or $O(d^{1+\varepsilon})$ with appropriate pre-conditioning) for any $\varepsilon > 0$, and it remains independent of the target function's information-exponent.
Our work departs from these results by showing that independence from the information-exponent can be achieved even without dependencies among input coordinates and with only a small perturbation of the isotropic distribution. Additionally, we obtain a linear sample complexity without pre-conditioning.
Beyond Gaussian data,~\cite{zweig2024single} examined the robustness of the standard `Gaussian isotropic picture' under perturbations affecting stability to linear projection and spherical symmetry. In contrast, our work investigates a type of perturbation to which the `Gaussian isotropic picture' is highly sensitive.




\paragraph{Sparse Boolean functions.} 
Previous works have shown that the sample and time complexity of learning sparse Boolean functions on uniform inputs using online SGD with squared loss on shallow networks depends on the hierarchical structure of the target function, measured by the \textit{leap} complexity~(\cite{abbe2022merged, abbe2023sgd}). More recently, \cite{joshi2024complexity} generalized this framework to arbitrary loss functions and product measures for a class of algorithms known as Differentiable Learning Queries (DLQ). They further proved that this complexity measure captures the learning dynamics of SGD on uniform inputs on two-layer networks in the mean-field regime and under linear scaling. Other works have shown that single monomials (i.e., sparse parities) under shifted inputs can be learned in linear time~(\cite{malach2021quantifying}), also with a neural network~(\cite{daniely2020learning}), making them more efficient to learn than under uniform inputs~(\cite{AS20, shalev2014understanding}). Additionally, studies have explored curriculum learning strategies that leverage easier samples to improve learning efficiency~(\cite{cornacchia2023mathematical, abbe2023provable}), limited to single monomials. Our work diverges from these prior studies in two key ways: (1) We prove that, with high probability over a random perturbation of the uniform distribution, most sparse functions have a leap complexity of one. (2) We show that this low leap complexity is sufficient to achieve linear time and sample complexity for learning sparse Boolean functions on a finite-width, two-layer network—thus extending beyond the mean-field regime and beyond uniform inputs. Beyond neural network learning, the complexity of learning sparse Boolean functions has also been extensively studied for general algorithms under uniform inputs and product distributions~(e.g. \cite{valiant2012finding,mossel2004learning}). 

\paragraph{Smoothed analysis.} \cite{spielman2001smoothed} introduced smoothed analysis as a hybrid between worst‑case and average‑case frameworks: one measures an algorithm’s expected performance under small random perturbations of arbitrary inputs. \cite{mossel2004learning} observed that shifting the input distribution alters its Fourier coefficients, potentially yielding drastic complexity improvements. This insight was later extended to smoothed‑analysis settings—first for learning juntas and decision trees under random product distributions~\cite{kalai2008decision} and then for learning DNF formulas~\cite{kalai2009learning}. \cite{brutzkus2020id3} showed that the ID3 decision‑tree algorithm can efficiently learn juntas in the smoothed model, and \cite{chandrasekaran2025learning} generalized these results to any Markov Random Field with a smoothed external field. In the Gaussian context, \cite{klivans2013moment} considered smoothing via additive Gaussian noise. Our work applies the same philosophy to gradient descent on shallow neural networks, yielding concrete, uniform runtime and sample‑complexity guarantees, in both discrete and continuous settings.

\section{Main results for Gaussian single-index models}
\label{sec:gaussian_results}
\subsection{Information-exponent of perturbed functions} 
Our main technical contribution is a quantitative study of the information-exponent under random perturbations. Before explaining the form of the perturbations and stating our results we first define the information-exponent: Let $\gamma$ stand for the law of $\cN(0,1)$, a standard Gaussian measure on $\bR$, and for $k \in \mathbb{N}$, let $H_k(x) := \frac{(-1)^k}{\sqrt{k!}} e^{\frac{x^2}{2}}\Big(e^{-\frac{x^2}{2}}\Big)^{(k)}$ stand for the $k^{\mathrm{th}}$ (normalized) Hermite polynomial. It is well known and easily seen through integration by parts, that $\{H_k\}_{k\in \bN}$ form a complete orthonormal system in $L^2(\gamma)$. Suppose now that $\|f\|_{L^2(\gamma)} < \infty$, this means that there exists a sequence $(\hat{f}(k))_{k\in \bN}$ such that $f (x)= \sum\limits_{k=0}^{\infty}\hat{f}(k)H_k(x)$, where the equality is understood in the sense of $L^2(\gamma)$ and where $(\hat{f}(k))_{k\in \bN}$ are the Hermite coefficients of $f$.
\bigskip

\begin{definition}[Information-exponent~(\cite{arous2021online}] \label{def:info_exponent}
    Let $f: \bR \to \bR$ be such that \\$\|f\|_{L^2(\gamma)} < \infty$ and let $(\hat{f}(k))_{k\in \bN}$ be its Hermite coefficients. The \emph{information-exponent} of $f$ is 
    \begin{align}
        \mathcal{I}(f):=\min\{k \geq 1| \hat{f}(k) \neq 0\}.
    \end{align}
\end{definition}
In words, $ \mathcal{I}(f)$ is the first non-vanishing Hermite coefficient of $f$, excluding $\hat{f}(0)$ which is the Gaussian expectation of $f$ and does not affect the optimization procedure.

The perturbations we consider are in the form of random shifts. Formally, instead of taking a sample of \emph{i.i.d.} standard normal $x_i\sim\mathcal{N}(0,\mathrm{I}_d)$, our sample is instead drawn from $x_i\sim\mathcal{N}(\alpha,\mathrm{I}_d)$ where $\alpha$ is a random mean which, for simplicity, we take to be Gaussian as well. The overall effect on the target is equivalent to shifting the target function $f$. In particular, if $\mu := \langle w^*,\alpha\rangle$, $f_\mu(x) = f(x+\mu)$, and $x_i - \alpha =:\tilde{x}_i \sim\mathcal{N}(0,\mathrm{I}_d)$, then,
$$y_i = f(\langle w^*,x_i\rangle) +\zeta_i= f(\langle w^*,\tilde x_i\rangle +\langle w^*,\alpha\rangle) +\zeta_i =f(\langle w^*,\tilde x_i\rangle +\mu) +\zeta_i = f_\mu(\langle w^*,\tilde x_i\rangle) +\zeta_i.$$
For a function $f$ we shall henceforth write $F_1(\langle w^*,\alpha\rangle)=F_1(\mu) := \hat{f}_\mu(1)$ for the first Hermite coefficient of $f_\mu$ \footnote{We remark that if $\mathcal{I}(f) > 1$, then $F_1(0) = 0$.}. In other words,
\begin{equation} \label{eq:hermitecoef}
F_1(\mu) = \E_{z \sim \mathcal{N}(0,1)}[f_\mu(z)\cdot H_1(z) ].
\end{equation} 
We can now state informally our main result concerning $F_1(\mu)$, see Section \ref{sec:smallball} for the formal statement and Assumption \ref{ass:nicef}, as well as the discussion that follows, for the exact conditions of the theorem. 
\vspace{0.1em}
\begin{theorem}[informal] \label{thm:infromalsmallball}
Let $\lambda > 0$ and let $f:\bR\to \bR$ satisfy some appropriate regularity assumptions (see Assumption \ref{ass:nicef}). Suppose that $\mu \sim \mathcal{N}(0,1)$. Then, there exists $c_\lambda\in (0,1)$, depending only on $\lambda$, such that
	$$\pr\left(|F_1(\mu)|>\lambda\right) > 1-c_\lambda.$$
\end{theorem}

Some remarks are in order concerning Theorem~\ref{thm:infromalsmallball}. First, from the expression \eqref{eq:hermitecoef} and by noting that when $\alpha \sim \mathcal{N}(0,\mathrm{I}_d)$ then $\mu \sim \mathcal{N}(0,1)$, it is unsurprising that, for a given function, any shifted version of it will $\gamma$-almost surely satisfy $F_1(\mu) \neq 0$. From this observation, we can immediately deduce an implicit expression of the form $\pr\left(|F_1(\mu)|>\lambda\right) > 1-c_\lambda$. The key point of Theorem \ref{thm:infromalsmallball} is that the estimates we derive are \emph{uniform} across a broad class of target functions. Consequently, very little specific information about the function is needed to establish the stated guarantees. This property will play a central part when designing optimization procedures for recovering $w^*$ and learning $f$, as these procedures can now be oblivious to the function $f$. Moreover, our estimates are fully quantitative and, as we explain below, are essentially optimal within the generality we consider. We derive these estimates by treating $F_1(\mu)$ as a functional in Gaussian space which then leads to examining small-ball probabilities of the function. To bound these probabilities, we extend $F_1$ to an analytic function in the complex plane, which enables us to apply known results, such as \cite{brudnyi2001distribution} or \cite{nazarov2003local}, about local small-ball probabilities of analytic functions. The local estimates are then transformed into global ones with appropriate concentration inequalities.

\subsection{Algorithmic implications} 
Given the results of \cite{arous2021online}, Theorem \ref{thm:infromalsmallball} has the following immediate implication. In the single-index model, hard functions are rather \emph{rare}, in a quantifiable sense, and \emph{most} functions have nearly linear sample complexity. That is, for every target function $f:\bR \to \bR$ the shifted version $f_\mu$ with small random $\mu$ satisfies that spherical online SGD, applied to $f_\mu$, will find the underlying signal $w^*$ with a nearly linear number of iterations, when initialized randomly on the sphere.

Since nature is inherently noisy, this result, on its own, may already explain some of the effectiveness of gradient-based algorithms as observed in practice. However, from an algorithmic perspective, this explanation is not entirely satisfactory. When learning $f$, we do not have access to the shift $\mu = \langle w^*, \alpha \rangle$, which depends on the unknown signal $w^*$.  
To further illustrate this point, suppose that we run SGD on $f$ with a sample drawn from $\mathcal{N}(\alpha,\mathrm{I}_d)$ and let $\theta_t$ be the iterates of SGD. Then, the shifts $\mu_t = \langle \theta_t,\alpha\rangle$ will change over time. So, with high probability, the target $f_\mu$ is easily learnable, but it is not clear that running SGD with the varying shifts $\theta_t$ will converge to the correct solution. Moreover, while we can bound the individual information exponents $\mathcal{I}(f_{\mu_t})$ along the trajectory of SGD, reasoning about the dynamics requires uniform bounds over time. This condition is not implied by Theorem \ref{thm:infromalsmallball}, especially when considering the dependencies between the different $\mu_t$'s. Below, we explain how to adapt various algorithms in different settings to circumvent this issue and handle shifted input distributions. We shall focus on the online setting, which is arguably harder, but mention that in principle one can apply similar modifications to other existing algorithms, e.g. multi-pass SGD. The common theme in our algorithmic results is as follows: by allowing for a shifted input distribution of the form $\mathcal{N}(\alpha, \mathrm{I}_d)$, one can obtain uniform complexity guarantees for learning $(f, w^*)$. In other words, the required sample size does not depend on the function $f$. 

At this point, we should remark that, given the nature of our approach, our algorithms below will either require access to a labeled sample of shifted inputs to query access to $f$. However, since our guarantees hold with high probability over the random
shift, we obtain the following perspective: For most directions, if the original labeled inputs are consistently shifted in this direction,
then algorithmic complexity remains independent of the target function. We believe this reflects a more realistic
setting than the standard assumption of purely isotropic inputs.

\paragraph{Parametric setting.}
We first consider the \emph{parametric} setting, as in \cite{arous2021online}. In this setting, the target function $f$ is known, and so we may initialize $\theta_0 \sim \mathrm{Uniform}(\bS^d)$ and run spherical SGD on the loss function $\mathcal{L}(\theta) = \frac{1}{n}\sum\limits_{i=1}^n\left(y_i - f(\langle \theta, x_i\rangle)\right)^2$. A crucial observation of \cite{arous2021online} is that the population loss $\E_\gamma [\mathcal{L}(\theta)]$ only has two types of critical points. The first critical point is the global minimizer, when $\theta = w^*$, which corresponds to the ground truth. The other type of critical point is the co-dimension $1$ sub-manifold of the equator $\{\theta|\langle \theta,w^*\rangle = 0\}$. Standard concentration of measure results dictate that $\theta_0$ must lie close to the equator. Hence, the main difficulty for SGD is escaping the equator and achieving non-trivial correlation with $w^*$, corresponding to \emph{weak learning}. Once weak learning is achieved there will be no obstacles for \emph{strong learning} and the iterates of SGD will converge rapidly to the global minimizer $w^*$. 
Capitalizing on this classification of critical points, we design a two-stage variant of SGD, see Algorithm \ref{alg:para} in Appendix~\ref{app:parametric_setting}. In the first stage, we use Theorem \ref{thm:infromalsmallball} and an appropriate shifted sample to facilitate weak learning. In the second stage, no extra modifications are required and we can run the standard SGD on isotropic inputs without shifting. 
We can thus prove the following.
\vspace{0.1em}
\begin{theorem}[Guarantees of Algorithm \ref{alg:para} in the parametric setting]\label{thm:parameteric}
   Let $f:\bR\to\bR$ satisfy some appropriate regularity assumptions (see Assumption~\ref{ass:nicef}) and let $w^*\in \bS^d$. Suppose that $\theta_0 \sim \mathrm{Uniform}(\bS^d)$ and that $\alpha \sim \cN(0,\mathrm{I}_d)$. Let $\theta$ be the output of Algorithm \ref{alg:para} and set $n = \Theta(d\ln^2(d))$. Then, for every $\lambda \in (0,\frac{1}{2})$, there is a choice of learning rate $\eta > 0$, which depends only on $f$ and $\lambda$, such that as $d \to \infty$,
    $$\pr\left(\langle \theta, w^*\rangle \geq 1 -o(1)\right) > \frac{1}{2}-\lambda,$$
    where above $n$ may depend on $\lambda$.
\end{theorem}

We first remark that the best probability guarantee to hope for is $\frac{1}{2}$, which by symmetry is the probability that $\langle \theta_0,w^*\rangle > 0$.  
Beyond that, as explained above, the main benefit of Theorem \ref{thm:parameteric} is that the sample complexity is almost independent of the function $f$. Previous results also demonstrated polynomial rates, but the power depended on some complexity measure of the function. In contrast, our algorithm performs with a nearly optimal complexity of $\tilde{O}(d)$, and the dependence on the target $f$ only appears as a multiplicative constant, rather than in the power. The proof of Theorem~\ref{thm:parameteric}, along with details on the algorithm considered, can be found in Appendix~\ref{app:parametric_setting}.

\paragraph{Semi-parametric setting.}
We next consider the \emph{semi-parametric setting}. In this case the target function $f$ is unknown. Thus the learning task combines a parametric task, learning $w^*$, and a non-parametric task, finding an approximation for $f$. A recent paper by Bietti, Bruna, Sanford, and Song~(\cite{bietti2022learning}) considered this setting and showed that the semi-parametric problem is solvable by following the gradient flow of the loss for a certain architecture of a shallow neural network, with ReLU activations. As before, the complexity of their algorithm depends on the information exponent of the unknown target function $f$. In Algorithm~\ref{alg:semiparam} in Appendix~\ref{app:semiparametric_setting}, we show how to adapt their construction to account for the variation of the shifted function along the gradient flow dynamics and prove the following guarantees.

\vspace{0.5em}

\begin{theorem}[Guarantees of Algorithm \ref{alg:semiparam} in the semi-parametric setting]\label{thm:semiparameteric}
    Let $f:\bR\to\bR$ satisfy some appropriate regularity assumptions (see Assumption \ref{ass:nicef}) and let $w^*\in \bS^d$. Suppose that $\theta_0 \sim \mathrm{Uniform}(\bS^d)$, that $\alpha \sim \cN(0,\mathrm{I}_d)$ and that all other parameters are set according to Assumption \ref{ass:init} and Theorem \ref{thm:nonpara}. If $\theta$ is the output of Algorithm \ref{alg:semiparam}, then for every $\lambda \in (0,\frac{1}{2})$, as long as $n = \Omega(d^2\log(d))$
    $$\pr\left(\langle \theta, w^*\rangle \geq 1 - o(1)\right) \geq \frac{1}{2} - \lambda,$$
    where above $n$ may depend on $\lambda$.
	Moreover, if $n = \Omega(d^3)$, and $N$ is the neural network obtained at the termination of the algorithm, we get that the width of $N$ is $O(\sqrt{\frac{n}{d^2}})$ and that
	$$\E_{x\sim \cN(0,1)}\left[(N(x) - f(\langle {w^*},x\rangle))^2\right] \leq \frac{1}{d^\beta},$$
	for some $\beta > 0$.
\end{theorem}
In Theorem \ref{thm:semiparameteric}, we establish two types of guarantee. First, when the sample complexity is quadratic, $n = \tilde{O}(d^2)$, we achieve strong recovery of the signal $w^*$ as $d \to \infty$, with a rather small network of logarithmic width. Second, when the complexity increases to $n = \Omega(d^3)$, we further obtain an $L^2$ guarantee on the efficiency of approximating $f$ using a shallow neural network. While these bounds are slightly sub-optimal, they could be improved by using a neural network with a smoother activation function; see the discussion in \cite[Appendix F]{bietti2022learning}. We chose this presentation both because of the widespread usage of the ReLU activation and because our aim was to obtain uniform complexity guarantees.

Again, we observe that the sample complexity is not particularly sensitive to the specific target function $f$, as the dependence appears only through multiplicative constants. This stands in contrast to the algorithms in \cite{dudeja2018learning} and \cite{bietti2022learning}, which exhibited polynomial rates dependent on either the information-exponent or certain smoothness parameters of $f$. The proof of Theorem~\ref{thm:nonpara}, along with details on the algorithm, can be found in Appendix~\ref{app:semiparametric_setting}.

\subsection{Beyond the single-index model}
\label{sec:beyond_single_index}
Thus far, we have focused on the single-index model, which represents one of the simplest examples of a model with a latent low-dimensional structure. Building on our results, it is natural to explore more general models that share this underlying principle. In particular, we now briefly discuss the \emph{multi-index model}, which can be seen as the immediate generalization of the previously considered setting.

In this model, instead of depending on a single direction, we have $k$ independent directions. So, the target function takes the form $f:\bR^k\to\bR$ and the signals are an orthonormal set $\{w^*_j\}_{j=1}^k$. Given a sample $(x_i)_{i=1}^n$ and $y_i = f(\langle w^*_1,x_i\rangle,\dots,\langle w^*_k,x_i\rangle) +\zeta_i$ the goal is to recover the subspace $\mathrm{span}(w^*_1,\dots,w^*_k)$ (note that the individual vectors are not necessarily identifiable in this problem).

In this setting, one of the key advantages of Theorem \ref{thm:infromalsmallball} is that the arguments are essentially dimension-free and rely on localization techniques, drawing from the celebrated small-ball estimates of~\cite{carbery2001distributional}. As a result, Theorem~\ref{thm:infromalsmallball} extends, with the necessary modifications, to this new model as well. Specifically, we can shift the input distribution with a random mean and still expect the first Hermite coefficients to remain bounded away from $0$, \emph{in any fixed direction}. However, unlike in the single-index model, the relationship between these coefficients and the optimization landscape is less straightforward, see \cite{bietti2023learning}. In particular, the convergence of gradient-based methods will depend on a sequence of non-vanishing Hermite coefficients, which are influenced by the underlying dynamics and which, in our setting, exhibit statistical dependencies. Moreover, unlike the single-index model, there are no universality results: there exists a target function for which the gradient flow will fail to find the global minimizer \cite[Theorem 4.3]{bietti2023learning}. 

Given the described differences concerning the multi-index model, it is worth exploring what type of algorithmic guarantees can be established and under which assumptions on the target function. Addressing this question in full generality is beyond the scope of the present work and is left as an interesting direction for future research. Instead, in the next Section, we focus on the specific case of Boolean sparse functions, which also serves to illustrate the general ideas. 






\section{Main results for sparse Boolean functions}
\label{sec:junta}
In this Section, we assume Boolean inputs $x \in \{ \pm 1 \}^d$ and a target function $f: \{ \pm 1 \}^d \to \bR$ that is $k$-sparse, i.e. it depends only on an unknown set of $k$ coordinates, with $k$ bounded:
\begin{align}
    f(x) = f(x_T), \qquad T \subseteq [d], |T| = k.
\end{align} 
For a choice of shifts $\boldsymbol{\mu }= (\mu_i)_{i \in [d]} \in [-1,1]^d$, let us define the shifted distribution as: 
\begin{align} \label{eq:Dmu}
    \mathcal{D}_{\boldsymbol{\mu}} : = \otimes_{i \in[d]} \Rad\left(\frac{\mu_i +1}{2} \right),
\end{align}
where $\Rad(p)$, $p \in [0,1]$ denotes the Rademacher distribution with parameter $p$ (specifically $z \sim \Rad(p)$ if and only if $ \pr(z=1) = 1-\pr(z=-1) = p$) and $\otimes_{i \in[d]}  $ denotes a product measure with independent coordinates.
Recall the Fourier-Walsh expansion of $f$ with orthonormal basis elements under $ \cD_{\bmu}$ (see~\cite{o'donnell_2014} for reference):
\begin{align} 
    f(x) = \sum_{S \subseteq [d]} \hat f_{\bmu}(S) \chi_{S,\bmu}(x),
\end{align}
where $\chi_{S,\bmu}(x): = \prod_{i \in S} \frac{x_i-\mu_i}{\sqrt{1-\mu_i^2}}$ are the basis elements and $\hat f_{\bmu}(S) := \E_{x \sim \cD_{\bmu}} [f(x) \chi_{S,\bmu}(x)]$ are the Fourier-Walsh coefficients of $f$. For brevity, we denote by $\hat f(S):= \hat f_{\boldsymbol{0}}(S) $ the Fourier-Walsh coefficients under the uniform distribution. We first show the following proposition, stating that with high probability over a random choice of $\bmu$, the first-order Fourier coefficients on the relevant coordinates of $f$ are of constant order (see also Lemma 2 and 3 in~\cite{kalai2008decision}). This is the analogous of Theorem~\ref{thm:infromalsmallball} in the Gaussian single-index case, with the Hermite coefficients replaced by the Fourier-Walsh coefficients.
\vspace{0.1em}
\begin{proposition} \label{prop:fourier_firstorder}
    Let $\bmu \sim  \Unif [-\eta, \eta]^{\otimes d}$, with $\eta \in [0,3/4]$, and let $\cD_{\bmu}$ be defined as in~\eqref{eq:Dmu}. Let $f$ be a sparse Boolean function, with support $T \subseteq [d]$ such that $|T|=k=O_d(1)$, and let $ \hat f_{\bmu} (S)$, for $S \subseteq [d]$ be its Fourier-Walsh coefficients under $\cD_{\bmu}$. Then, for all $j \in T$ and for $\varepsilon \in (0,1)$:
    \begin{align}
        \pr \left(| \hat f_{\bmu}(\{ j \} ) | < \varepsilon \eta^{k-1} \sqrt{\Inf_j(f)} \right) \leq O(\varepsilon^{1/k}),
    \end{align}
    where $ \Inf_j(f) = \sum_{S: j \in S} \hat f(S)^2$ is the Boolean influence of coordinate $j$.
\end{proposition} 
Proposition~\ref{prop:fourier_firstorder} states that if $\eta$ is of constant order, then with high probability, the first-order Fourier coefficients on the support coordinates (i.e., those for which $\Inf_j(f) > 0$) are also of constant order, with the scaling $\eta^k$ appearing as a worst-case bound over all $k$-sparse $f$. This highlights the importance of the sparsity assumption on $f$ and the need for $k$ to remain bounded.
The proof follows by application of the Carbery-Wright inequality~(\cite{carbery2001distributional}), and can be found in Appendix~\ref{app:junta_proofs}.




We next show that having large first-order Fourier coefficients on the support coordinates is sufficient to guarantee learning with $O(d)$ samples (up to logarithmic factors) for any $k$-sparse Boolean function, with inputs coming from the chosen $\cD_{\bmu}$. For this, we put ourselves in a favorable setting for theoretical analysis. We consider a two-layer neural network $\NN(x;\theta) := \sum_{i=1}^N a_i \sigma(w_ix+b_i)$, were $\sigma(t) = \max(t,0)$ is the ReLU unit and $\theta = (a_i,w_i,b_i)_{i=1}^N \subset (\bR\times \bR^d\times \bR)^n$ is the set of trainable parameters. In particular, we consider layer-wise training, where the first layer is trained for one step by SGD with
the covariance loss (used also in~\cite{cornacchia2023mathematical,abbe2023provable}, see Appendix~\ref{app:junta_proofs} for a definition) and with constant learning rate. After this first step, we train the second layer until convergence with any convex loss function. We do not train the bias neurons, which are sampled uniformly at random from an appropriate interval and kept frozen. Both layerwise training and frozen bias neurons are common choices in theoretical investigations of neural networks learning~(\cite{daniely2020learning,abbe2023sgd,barak2022hidden,lee2024neural}). Our theorem reads as follows.


\vspace{0.5em}
\begin{theorem} \label{thm:juntas_formal}
    Let $f$ be a $k$-sparse function. For $\varepsilon,\delta>0$, there exists $c>0$ such that with probability $1-O(\varepsilon^c)$, layerwise-SGD with batch size $B = \tilde \Omega(d)$\footnote{where $\tilde \Omega(d^c) = \Omega(d^c \poly \log(d))$ for any $c \in \bR$.} with the covariance loss on a $2$-layer ReLU network of size $\tilde \Omega(\eta^{-(k+1)} d \varepsilon^{-1})$ after $T= \tilde \Omega(\delta^{-2} \varepsilon^{-2}\eta^{-2(k+1)})$ training steps learns $f$ up to error $\delta$.
\end{theorem}
Theorem~\ref{thm:juntas_formal} states that if $\eta$ is of constant order, then with high probability over the choice of $\bmu$ and the algorithm's randomness, layerwise-SGD can learn any $k$-sparse function with linear sample complexity, independently on the target $f$ and its leap complexity~(\cite{abbe2023sgd}). If $\eta=o_d(1)$, the complexity scales as $\eta^{-\Omega(k)}$, and for $\eta $ small enough
we retrieve the bounds on uniform inputs, with $O(d^{k-1})$ complexity in the worst-case over all $k$-sparse functions~(\cite{abbe2023sgd,kou2024matching}). This highlights the importance of having some randomness in the distribution. In particular, we remark that for all deterministic choices of $\bmu$, there exist $k$-sparse targets that may require sample complexity strictly larger than $\tilde \Omega(d)$-for instance, the basis elements $ \chi_{T, \bmu}(x)$, with $|T|=k$. However, such complex targets are rare and highly sensitive to small shifts in the input distribution. 
We defer the proof of Theorem~\ref{thm:juntas_formal} to Appendix~\ref{app:junta_proofs}.

\section{Small-ball estimates for Hermite coefficients} \label{sec:smallball}

In this section, we formalize Theorem \ref{thm:infromalsmallball} in a fully quantitative way. Before stating the result we first explain and discuss the necessary assumptions on the target function.
Let $f:\bR\to \bR$, and recall the definition of the shifted Hermite coefficient, 
\begin{align*}
F_1(\mu) := \E_{x\sim \cN(\mu, 1)} \left[f(x) (x-\mu)  \right] = \E_{x\sim \cN(0, 1)} \left[f(x+\mu) x \right].
\end{align*}
We shall henceforth enforce the following assumptions on $f$. 
\vspace{0.5em}
\begin{assumption} \label{ass:nicef}
        Let $f:\bR\to \bR$. We assume
	$f$ satisfies the following properties:
	\begin{enumerate}
		\item {\bf Normalization.} $\E_{x\sim \mathcal{N}(0,1)}[f] = 0$ and $\E_{x\sim \mathcal{N}(0,1)}[f^2] = 1$.
		\item {\bf Regularity.} $f$ is $L$-Lipschitz, for some $L \geq 1$. In particular, by Rademacher's theorem, $f$ is differentiable almost everywhere.
		\item {\bf Non-linearity.} $f'$ has a distributional (weak) derivative. Moreover, there exists $\varepsilon,\delta > 0$ and $c \in [-1,1]$ such that 
        for every $\frac{\delta}{2} <s<\delta$, 
        \begin{equation} \label{eq:nonlinear}
            \left|\int\limits_{c-s}^{c+s} f''(t)dt\right| > \varepsilon.
        \end{equation}
	\end{enumerate}
\end{assumption}
Before stating our result, we first discuss the role of the assumptions. The assumption regarding \emph{non-linearity} is perhaps the most non-standard. However, note that it is satisfied by any twice (weakly) differentiable non-linear function, with some set of parameters. Specifically, if $f$ is twice continuously differentiable, this assumption simply requires that for some $c \in \mathbb{R}$, $f''(c) \neq 0$. Similar assumptions exist in the literature, such as the non-degeneracy assumption of \cite{damian2022neural}. Our assumption is more general though, permitting certain derivative discontinuities and applying, for example, to piecewise-linear functions. 
It is also worth noting that the requirement for $c \in [-1,1]$ is not essential; the exact location of $c$ could, in principle, be arbitrary, though this would introduce an additional parameter to track and encumber the proof. We remark as well that if $f$ is linear and non-constant, then $F_1$ does not depend on $\mu$ and $F_1(0) \neq 0$. Hence, since these functions are also less interesting from the optimization perspective, we can safely disregard them.

The \emph{regularity} requirement is a standard assumption in this setting; however, we do mention that it could be significantly relaxed. It is readily seen that this assumption implies that $F_1(\mu)$ is also a Lipschitz function, of $\mu$. When $\mu$ follows a Gaussian distribution, we leverage this property alongside the Gaussian isoperimetric inequality to establish a small-ball estimate for $F_1(\mu)$. As will become evident in the proof, this step remains valid even if $F_1(\mu)$ is only \emph{locally Lipschitz}. Establishing local Lipschitz bounds on $F_1$ requires considerably weaker assumptions on $f$: as long as $f$ exhibits exponential or even tame sub-Gaussian growth, this condition is automatically satisfied. So, while our result can be extended to a much larger class of functions, we chose this familiar assumption mainly for simplicity and since together with the non-linearity assumption already encompasses many cases of interest, such as ReLU and the sigmoid.

With Assumption \ref{ass:nicef} in place, our main result concerning $F_1$ is the following.

\vspace{0.5em}
\begin{theorem} \label{thm:quantitativeshift}
        Suppose that $\mu \sim \mathcal{N}(0,1)$ and let $f:\bR \to \bR$ satisfy Assumption \ref{ass:nicef}. Then, there exists a constant $c = c(\varepsilon, \delta, L)$, such that for any $\lambda>0$ small enough, 
	$$\pr\left(|F_1(\mu)| \leq \lambda\right) \leq \exp\left(-c\log\left(\frac{1}{\lambda}\right)^{\frac{2}{3}}\right).$$
\end{theorem}
We first mention that the proof of Theorem \ref{thm:quantitativeshift} also works if $\mu \sim \mathcal{N}(0,\eta)$, for $\eta \in (0,1)$. The main difference will be in the value of the constant $c$ which will also depend, in a somewhat complicated way, on the value of $\eta$.
We next discuss the super-polynomial bound afforded by Theorem \ref{thm:quantitativeshift}. In some sense, this bound is the best we could hope for at this level of generality. Indeed, if $f$ is a degree-$k$ polynomial, then so is $F_1$ and we would expect the bound to be of the form $O\left(\lambda^{\frac{1}{k}}\right)$, as in Proposition \ref{prop:fourier_firstorder}. On the other hand, $\exp\left(-c\log\left(\frac{1}{\lambda}\right)^{\frac{2}{3}}\right)$ is only `barely super-polynomial', and while it is plausible that the power of the logarithm could be improved to something in $(\frac{2}{3},1)$ not much could be improved beyond that. 
The proof of Theorem \ref{thm:quantitativeshift} is conducted in two stages and appears in Appendix \ref{sec:appendixsmallball}. In Section \ref{sec:weaksmall} we will use Assumption \ref{ass:nicef} to first prove a weak estimate, for a single fixed value of $\lambda$. After establishing this weak estimate, in Section \ref{sec:generalsmallball} we shall show that $F_1$ extends to an entire function of the complex plane and apply local small-ball estimates for such functions.

\section{Numerical experiments}
\begin{figure}[t]
    \centering
    \includegraphics[width=0.49\linewidth]{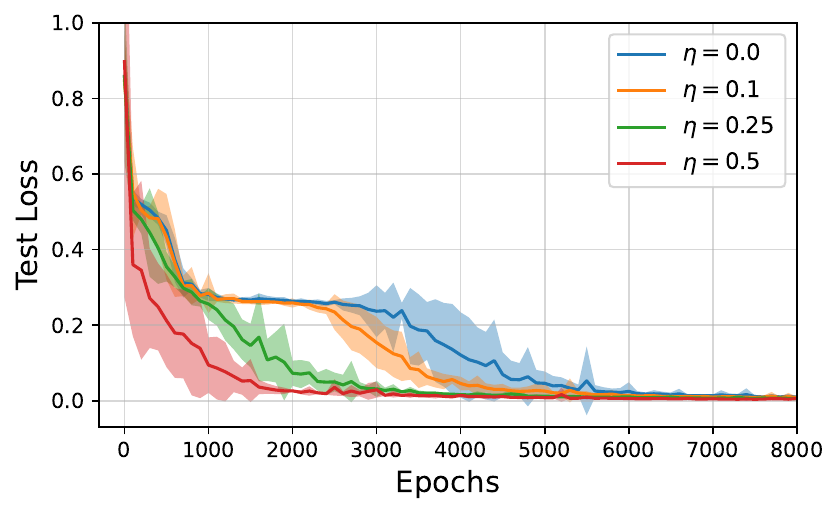}
    \includegraphics[width=0.49\linewidth]{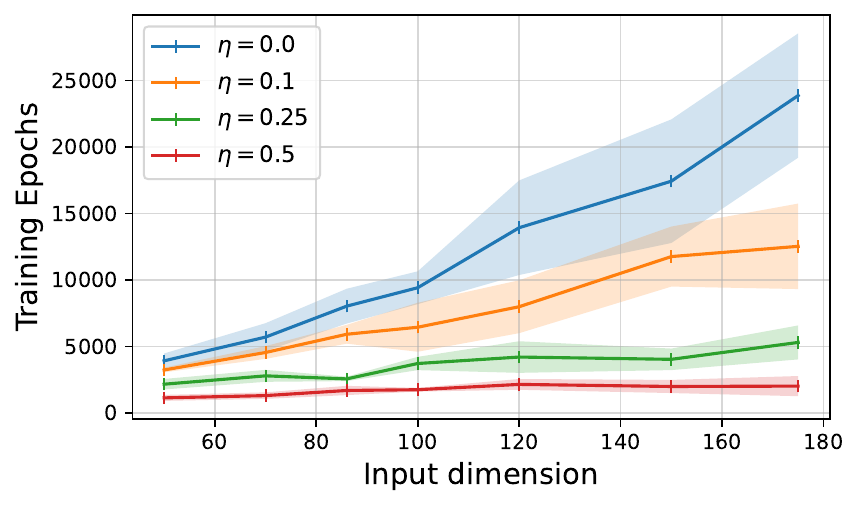}
    \caption{Learning the Boolean function $f(x) = x_1+x_1x_2x_3+x_1x_2...x_6$ using SGD on a two-layer $\ReLU$ network under randomly shifted Boolean inputs, with shift-magnitude $\eta \in \{ 0,0.1,0.25,0.5 \}$. Left: Evolution of the test error during training for fixed input dimension $d=50$. Right: Number of training epochs required to achieve test error below $10^{-2}$ for varying input dimensions.}
    \label{fig:sparseBoolean}
\end{figure}
Given our theoretical results, it is instructive to observe how our predictions hold in practice. In Figure~\ref{fig:sparseBoolean}, we consider learning a $6$-sparse Boolean function defined by 
$$f(x) = x_1+x_1x_2x_3+x_1x_2x_3x_4x_5x_6,$$
where $x \in \{ \pm 1 \}^d$. We consider different shifted input distributions $\cD_{\bmu}$ (as defined in~\eqref{eq:Dmu}), where $\bmu \sim \Unif[-\eta, \eta]^{\otimes d}$ and $\eta \in \{ 0,0.1,0.25,0.5\}$. We take a two-layer $\ReLU$ network with $512$ hidden units. We train it with online mini-batch SGD with batch size $64$, with squared loss, and with all weights and biases trained jointly. We repeat each experiment $10$ times. Figure~\ref{fig:sparseBoolean}(left) shows the evolution of the test error during training, for a fixed input dimension $d=50$. As $\eta$ increases, learning accelerates, requiring fewer training epochs. As it can be seen in the no-noise regime, when $\eta=0$, the first order Fourier coefficient on the first coordinate $\hat{f}(\{1\}) > 0$, and we observe a sharp decrease in training loss near initialization. On the other hand, for the other coordinates in the support, $\hat{f}(\{2\})=\dots=\hat{f}(\{6\})=0$, and so the dynamics eventually encounter saddle points, leading to a stagnation in the test loss. In contrast, as predicted by Theorem \ref{thm:juntas_formal}, when introducing noise these effects gradually disappear. As $\eta$ increases, the dynamics avoid the observed plateaus, and the test error rapidly converges to zero. In Figure~\ref{fig:sparseBoolean}(right), we plot the number of training epochs required to achieve test error below $10^{-2}$, for different input dimensions $d \in \{ 50,70,86,100,120,150,175 \}$. We observe that as $\eta$ increases, the scaling of the training time with respect to the input dimension decreases.

\section{Conclusion}
In this paper, we analyze the complexity of learning Gaussian single-index and sparse Boolean functions under randomly shifted input distributions. In both cases, we demonstrated that a random shift in the first moment ensures that the first-order coefficients in the relevant expansion remain of constant order, making learning efficient and independent of the specific target function. This suggests that when a low-dimensional structure is present, most target functions are easy to learn, with high-complexity cases being rare. A natural next step is to extend our analysis from Gaussian single-index functions to the multi-index setting. While our results for sparse Boolean functions suggest that similar behavior may hold in the multi-index case, fully addressing this question remains challenging. As discussed in Section~\ref{sec:beyond_single_index}, the small-ball estimates for Hermite coefficients can be extended to this setting, since Theorem~\ref{thm:infromalsmallball} is essentially dimension-free. However, obtaining general algorithmic guarantees remains non-trivial.
Another promising direction is to explore scenarios where input coordinates are correlated, shifting the low-dimensional structure from the target function to the input distribution itself.

\section{Acknowledgments}
EC was supported by the NSF DMS-2031883, Bush Faculty Fellowship ONR-N00014-20-1-2826, and the French government under management of Agence Nationale de la Recherche as part of the “Investissements d’avenir” program, reference ANR19-P3IA-0001 (PRAIRIE 3IA Institute). DM was partially supported by a Simons Investigator award (622132, PI: Mossel) and the  Brian and Tiffinie Pang Faculty Fellowship. EM was partially supported by 
NSF DMS-2031883, Bush Faculty Fellowship ONR-N00014-20-1-2826 and Simons Investigator award (622132). 



\bibliography{references}
\bibliographystyle{alpha}

\newpage
\appendix

\section{Algorithmic Results}
\subsection{Parametric setting}
\label{app:parametric_setting}
The solution to the parametric problem is divided into two steps and is presented in Algorithm \ref{alg:para}.  At initialization, since the initial guess is uniformly distributed on the sphere, it will, with high probability, lie near the equator $\{\theta \mid \langle \theta, w^* \rangle = 0\}$. This region corresponds to a flat area in the optimization landscape of SGD with respect to Gaussian inputs, where gradients offer less information about the signal. To address this point, the first step is geared towards escaping the flat region. It does so by implementing a single batched SGD update on a shifted distribution, enabling the algorithm to break potential symmetry. Below we explain more about this step. Recall that being far from the equator is equivalent to weak recovery of the signal $w^*$. That is, to finding a vector in $\bS^d$ with non-negligible correlation with $w^*$
Once this is achieved, the second step involves standard SGD dynamics on isotropic data, which boosts weak recovery to strong recovery. This boosting procedure was already addressed in \cite{ba2022high}, so we will mostly focus on the first step in this section.
\vspace{0.5em}

\begin{algorithm}[H]
\caption{Two-step solution to the parametric problem}
\label{alg:para}
\LinesNumbered
\KwIn{Target function $f:\bR\to\bR$, initial guess $\theta_0 \in \bS^d$, shift $\alpha\in \bR^d$,\\
Gaussian sample $\{\tilde{x}_i\}_{i=1}^{2n},$ learning rate $\eta > 0$.}

\BlankLine
\textbf{Step 1: Escaping the equator}

Set $x_i = \tilde{x}_i + \alpha$ for $i=1,\dots,n$.

Get noisy labels $y_i = f(\langle w^*, x_i\rangle) + \zeta_i$ for $i=1,\dots,n$.

Set $V = \nabla^S\mathcal{L}^n_{\mu,\mu^*}(\theta_0, (x_i)_{i=1}^n)$, according to \eqref{eq:empirloss}.

Set $\tilde\theta_1 = \theta_0 \pm \eta V$.

Set $\theta_1 = \frac{\tilde{\theta_1}}{\|\tilde \theta_1\|}$.

\BlankLine
\textbf{Step 2: From weak learning to strong learning}

Get noisy labels $y_i = f(\langle w^*, \tilde x_i\rangle) + \zeta_i$ for $i=n+1,\dots,2n$.

\For{$t = 1$ \KwTo $n$}{
    Set $V_t = \nabla^S\mathcal{L}(\theta_{t}, \tilde x_{n+t})$.
    
    Set $\tilde{\theta}_{t+1} = \theta_{t} - \log(d)^{-\frac{3}{2}}V_t$.
    
    Set $\theta_{t+1} = \frac{\tilde{\theta}_{t+1}}{\|\tilde \theta_{t+1}\|}$.
}

\Return $\theta_{n+1}$
\end{algorithm}

\paragraph{Preliminaries:}
Let $\alpha \in \bR^d$ and for $\theta_0,w^* \in \bS^d$ set $\mu= \langle\theta_0,\alpha\rangle$ and $\mu^*= \langle w^*,\alpha\rangle$ . Recall that for $f:\bR\to\bR$ we have $f_\mu(x)=f(x+\mu)$. Thus, if $\tilde{x} \sim \cN(0,\mathrm{I}_d)$ and  $\tilde{x} + \alpha =:x\sim \cN(\alpha,\mathrm{I}_d)$, we have the following expression for the \emph{quadratic population loss}, 
$$\mathcal{L}^p(\theta_0)=\E\left[(f(\langle\theta_0,x\rangle)-f(\langle w^*,x\rangle))^2\right] + \E\left[\zeta^2\right]= \E\left[(f_\mu(\langle\theta_0,\tilde{x}\rangle)-f_{\mu^*}(\langle w^*,\tilde{x}\rangle))^2\right] +\E\left[\zeta^2\right].$$
From now on we shall regard $\mu$ and $\mu^*$ as two fixed numbers and write $\mathcal{L}_{\mu,\mu^*}^p(\theta_0)$ for the right-hand side of the above equation, as a function of $\theta_0$ only. We stress the fact that while $\mu$ is also a function of $\theta_0$, Step 1 consists of a single step of gradient descent, and so we proceed in that step without considering any potential changes to $\mu$.  With this comment in mind, we now express the loss in terms of the Hermite coefficients of $f_\mu$ and $f_{\mu^*}$ as well as the overlap $m(\theta_0) = \langle \theta_0,w^*\rangle$. 
For the standard orthonormal Hermite polynomials $\{H_k\}_{k\geq 0}$ write 
$$f_\mu = \sum\limits_{k\geq0} \hat{f}_\mu(k)H_k \text{ and } f_{\mu^*} = \sum\limits_{k\geq0} \hat{f}_{\mu^*}(k)H_k.$$
If $G$ and $G'$ are two standard Gaussians with $\E\left[GG'\right] = \rho$, then $\E\left[H_k(G)H_k(G')\right] = \rho^k$. So,
$$\mathcal{L}_{\mu,\mu^*}^p(\theta_0)=\E_{\mathcal{N}(0,1)}\left[f^2_\mu\right] + \E_{\mathcal{N}(0,1)}\left[f^2_{\mu^*}\right]-2\sum\limits_{k\geq 0}\hat{f}_\mu(k)\hat{f}_{\mu^*}(k)m(\theta_0)^k,$$
and 
\begin{equation} \label{eq:popderivtaive}
\nabla^S\mathcal{L}_{\mu,\mu^*}^p(\theta_0)=-2\left(\sum\limits_{k\geq 1}k\hat{f}_\mu(k)\hat{f}_{\mu^*}(k)m(\theta_0)^{k-1}\right)\nabla^Sm(\theta_0),
\end{equation}
where the spherical gradient is with respect to $\theta_0$, and where we treat $\mu$ as a fixed number, according to our comment above.
As we shall see below, starting from $\theta_0$, one update step in the direction defined by \eqref{eq:popderivtaive} is enough to guarantee weak learning. However, since we do not have direct access to the quantities $\hat{f}_{\mu_*}(k)$ or $m(\theta_0)$ we will need to find an appropriate estimator for the population loss.
For this estimation task we shall use the \emph{empirical loss}: let $(\tilde{x}_i)_{i=1}^n$ be \emph{i.i.d.} $\cN(0,\mathrm{I}_d)$, so that $x_i := \tilde{x}_i +\alpha $ are \emph{i.i.d.} $\cN(\alpha,\mathrm{I}_d)$. For $y_i = f(\langle w^*,x_i\rangle) + \zeta_i = f_{\mu^*}(\langle w^*,\tilde x_i\rangle) + \zeta_i$ the batched empirical loss is given by
\begin{align} \label{eq:empirloss}
    \mathcal{L}^n_{\mu,\mu^*}&(\theta_0,(x_i)_{i=1}^n) = \frac{2}{n}\sum\limits_{i=1}^n (f_\mu(\langle \theta_0, \tilde{x}_i\rangle) - y_i)^2 \nonumber\\
    &\implies \nabla^S \mathcal{L}^n_{\mu,\mu^*}(\theta_0,(x_i)_{i=1}^n)=\frac{1}{n}\sum\limits_{i=1}^n f'_\mu(\langle \theta_0, \tilde{x}_i\rangle)(f_\mu(\langle \theta_0, \tilde{x}_i\rangle) - y_i)\nabla^S\langle\theta_0,\tilde x_i\rangle.
\end{align}
Note that under Assumption \ref{ass:nicef} we can differentiate under the integral sign and so 
\begin{equation} \label{eq:unbiased}
    \E\left[\nabla^S\mathcal{L}^n_{\mu,\mu^*}(\theta_0,(x_i)_{i=1}^n)\right] = \nabla^S\mathcal{L}_{\mu,\mu^*}^p(\theta_0).
\end{equation}
With standard concentration arguments, we further show that the batched empirical loss is a good approximation for the population loss.
\vspace{0.5em}
\begin{lemma} \label{lem:approxgrad}
    Suppose that $f:\bR\to\bR$ satisfies Assumption \ref{ass:nicef} and that $\zeta_i$ is $1$-sub-Gaussian. Then, for every $\theta_0 \in \bS^d$, and $\beta > 0$
    $$\pr\left(\left\|\nabla^S\mathcal{L}^n_{\mu,\mu^*}(\theta_0,(x_i)_{i=1}^n) - \nabla^S\mathcal{L}^p_{\mu,\mu^*}(\theta_0)\right\| \geq CL^2\sqrt{\frac{d\beta}{n}}\sqrt{\ln\left(dn\right)}\right) \leq \frac{1}{d^\beta}.$$
\end{lemma}
\begin{proof}
    We claim that when $\tilde{x}\sim \cN(0,\mathrm{I}_d)$, the random variable $Z:=f'_\mu(\langle \theta_0, \tilde{x}\rangle)(f_\mu(\langle \theta_0, \tilde{x}\rangle) - y_i)\nabla^S\langle\theta_0,\tilde x\rangle$ has sub-exponential tails. Indeed, from Assumption \ref{ass:nicef} $f$ is $L$-Lipschitz, and so are $f_\mu$ and $f_{\mu^*}$. So, $|f'_\mu(\langle \theta_0, \tilde{x}\rangle)| \leq L$ almost surely. Moreover, since $\langle\theta,\tilde{x}\rangle$ is a standard Gaussian, as a Lipschitz function $f_\mu(\langle \theta_0, \tilde{x}\rangle)$ is $L$-sub-Gaussian~(\cite{ledoux1993inegalite}). Together with the assumption on the noise we get that $(f_\mu(\langle \theta_0, \tilde{x}\rangle) - y_i)$ is centered and $3L$-sub-Gaussian. As $\|\nabla^S\langle \theta_0, \tilde{x}\rangle\| \leq \|\tilde{x}\|$, by the Cauchy-Schwartz inequality, we get for every $m \geq 2$
    $$(\E\left[|Z|^m\right])^{\frac{1}{m}}\leq C'\sqrt[2m]{\E\left[(f_\mu(\langle \theta_0, \tilde{x}\rangle) - y_i)^{2m}\right]\cdot\E[\|x\|^{2m}]} \leq C'L^2\sqrt{d}m,$$
    where we have used that if $Y$ is $\sigma$-sub-Gaussian then $\E\left[Y^{2m}\right]^{\frac{1}{2m}}\leq C'\sqrt{2m}\sigma,$ for some universal constant $C' > 0$. We thus conclude, that 
    $Z$ is $C'L^2d$-sub-exponential. So if $\{Z_i\}_{i=1}^n$ are \emph{i.i.d.} copies of $Z$, we get by Bernstein's inequality (see for example \cite[Proposition 7]{maurer2021concentration}) that,
    $$\pr\left(\left\|\frac{1}{n}\sum\limits_{i=1}^nZ_i - \E[Z]\right\| \geq CL^2\sqrt{\frac{d\beta}{n}}\sqrt{\ln\left(d\right)}\right) \leq \frac{1}{d^\beta}.$$
    The proof concludes with the identity in \eqref{eq:unbiased} and since $\tilde{x}_i$ are independent.
\end{proof}
Recall that $\theta_1 = P_{\bS^d}\left(\theta_0 - \eta\nabla^S\mathcal{L}^n_{\mu,\mu^*}(\theta_0,(x_i)_{i=1}^n)\right)$ where $P_{\bS^d}$ is the projection to $\bS^d$, and $\eta>0$ is a learning rate. Our next result shows that when the approximation error from Lemma \ref{lem:approxgrad} and $m(\theta_0)$ are small, one step of batched SGD will yield a significant correlation with the signal $w_*$, provided that $\hat{f}_\mu(1)\hat{f}_{\mu^*}(1)$ is bounded away from $0$. The quantitative bounds are chosen with some foresight, for the proof of Theorem \ref{thm:parameteric}, and could in principle be generalized.
\vspace{0.5em}
\begin{lemma} \label{lem:weaklearning}
    Suppose that $0<m(\theta_0)\leq \frac{1}{d^{\frac{1}{4}}}$, and that $\left\|\nabla^S\mathcal{L}^n_{\mu,\mu^*}(\theta_0,(x_i)_{i=1}^n) - \nabla^S\mathcal{L}^p_{\mu,\mu^*}(\theta_0)\right\|\leq \frac{1}{\sqrt{\log(d)}}$. If $\eta = \sqrt{\frac{\hat f_{\mu}(1)\hat f_{\mu^*}(1)}{90L^6}}$, then
    $$m(\theta_1) \geq \frac{2}{\sqrt{90}}\hat f_{\mu}(1)\hat f_{\mu^*}(1)\sqrt{\frac{\hat f_{\mu}(1)\hat f_{\mu^*}(1)}{90L^6}} -\frac{102L^4}{\sqrt{\log(d)}}.$$
\end{lemma}
\begin{proof}
    Write $\mathcal{E} := \left\|\nabla^S\mathcal{L}^n_{\mu,\mu^*}(\theta_0,(x_i)_{i=1}^n) - \nabla^S\mathcal{L}^p_{\mu,\mu^*}(\theta_0)\right\|$ and $r:= \|\theta_0 - \eta\nabla^S\mathcal{L}^n_{\mu,\mu^*}(\theta_0,(x_i)_{i=1}^n)\|$.
    So,
    \begin{equation} \label{eq:mincrement}
        m(\theta_1) \geq \frac{1}{r}\left(m(\theta_0)-\eta\langle\nabla^S\mathcal{L}^p_{\mu,\mu^*}(\theta_0),w^*\rangle - \eta\mathcal{E}\right).
    \end{equation}
    Also, since $\theta_0$ and $\eta\nabla^S\mathcal{L}^n_{\mu,\mu^*}(\theta_0,(x_i)_{i=1}^n)$ are, by definition, orthogonal, we get that 
    \begin{equation} \label{eq:rbound}
        r \leq \sqrt{1+\eta^2\|\nabla^S\mathcal{L}^n_{\mu,\mu^*}(\theta_0,(x_i)_{i=1}^n)\|^2} \leq 1+\eta^2\|\nabla^S\mathcal{L}^n_{\mu,\mu^*}(\theta_0,(x_i)_{i=1}^n)\|^2\leq 1 + 8\eta^2L^4 + 2\eta^2\mathcal{E}^2.
    \end{equation}
    Above we have used the expression \eqref{eq:popderivtaive} according to which,
    $$\|\nabla^S\mathcal{L}_{\mu,\mu^*}^p(\theta_0)\|\leq 2\sqrt{\sum\limits_{k\geq 1}k\hat{f}^2_\mu(k)\sum\limits_{k\geq 1}k\hat{f}^2_{\mu^*}(k)} \leq 2\max\limits_{x\in \bR}|f'_\mu(x)|\max\limits_{x\in \bR}|f'_{\mu^*}(x)|\leq 2L^2,$$
    where the second inequality is a standard calculation with Hermite polynomials (see for example the proof of \cite[Proposition 2.1]{ba2022high}).
    Combining \eqref{eq:mincrement} with \eqref{eq:rbound} along with the elementary inequality $|\frac{1}{1+t}-1| \leq t$ valid for $t>0$, we arrive at
    \begin{align} \label{eq:m1bound}
        m(\theta_1) \geq m(\theta_0) &-\eta\langle\nabla^S\mathcal{L}^p_{\mu,\mu^*}(\theta_0),w^*\rangle - \eta\mathcal{E}- \eta^2\left(8L^4+2\mathcal{E}^2\right)|m(\theta_0)|\nonumber\\
        &-\eta^3\left(8L^4+2\mathcal{E}^2\right)\left(|\langle\nabla^S\mathcal{L}^p_{\mu,\mu^*}(\theta_0),w^*\rangle| +\mathcal{E}\right).
    \end{align}
    Since $0\leq m_0,\mathcal{E} \leq \frac{1}{\sqrt{\log(d)}}$, we now get
    $$m(\theta_1) \geq -\eta\langle\nabla^S\mathcal{L}^p_{\mu,\mu^*}(\theta_0),w^*\rangle - \eta^310L^4\langle\nabla^S\mathcal{L}^p_{\mu,\mu^*}(\theta_0),w^*\rangle| - \frac{100L^4}{\sqrt{\log(d)}}.$$
    To finish the proof we focus on estimating $\langle\mathcal{L}^p_{\mu,\mu^*}(\theta_0),w^*\rangle$. First, since $\nabla^S m(\theta_0) =\nabla^S(\langle \theta_0,w^*\rangle) = w^* - \langle \theta_0, w^*\rangle \theta_0,$ we have
    $$\frac{1}{2}\leq 1- |\langle \theta_0, w^*\rangle|^2 \leq \langle \nabla^Sm(\theta_0), w^*\rangle \leq 1 +|\langle \theta_0, w^*\rangle|^2 \leq \frac{3}{2}.$$
    Furthermore, with the same argument as above we have
    $$\left|\sum\limits_{k\geq 2} k\hat f_{\mu}(k)\hat f_{\mu^*}(k)m(\theta_0)^{k-1}\right| \leq m(\theta_0)L^2 \leq \frac{L^2}{d^{\frac{1}{4}}}.$$
    Thus,
    $$-\langle\nabla^S\mathcal{L}^p_{\mu,\mu^*}(\theta_0),w^*\rangle \geq \hat f_{\mu}(1)\hat f_{\mu^*}(1) - 2\frac{L^2}{\sqrt{\log(d)}},$$
    and
    $$\left|\langle\nabla^S\mathcal{L}^p_{\mu,\mu^*}(\theta_0),w^*\rangle\right| \leq 3L^2.$$
    Plugging this into \eqref{eq:m1bound} we get
    $$m(\theta_1) \geq \eta f_\mu(1)f_{\mu^*}(1) - 30L^6\eta^3-\frac{102L^4}{\sqrt{\log(d)}}.$$
    If we now optimize and choose $\eta = \sqrt{\frac{f_{\mu}(1)f_{\mu^*}(1)}{90L^6}}$, we get
    $$m(\theta_1) \geq \frac{2}{\sqrt{90}}f_{\mu}(1)f_{\mu^*}(1)\sqrt{\frac{f_{\mu}(1)f_{\mu^*}(1)}{90L^6}} -\frac{102L^4}{\sqrt{\log(d)}}.$$
\end{proof}
Combining Lemma \ref{lem:approxgrad} and Lemma \ref{lem:weaklearning} we can now show that by appropriately choosing the sample size, we can achieve weak recovery of the signal $w^*$. Theorem \ref{thm:quantitativeshift} will provide the appropriate bounds on the derivative term $\hat f_\mu(1)\hat f_{\mu^*}(1).$ The only thing that is missing is showing that the second step of the algorithm allows for strong recovery of $w^*$. That part was already established in \cite{ba2022high}, and so we can now prove Theorem \ref{thm:parameteric}.

\vspace{1em}
\begin{proof}[Proof of Theorem \ref{thm:parameteric}]
    We begin with the analysis of Step $1$ of the algorithm. By Theorem \ref{thm:quantitativeshift} we have that $\hat{f}_\mu(1)\hat{f}_{\mu^*}(1) \neq 0$. We can further assume, with no loss of generality that $\hat{f}_\mu(1)\hat{f}_{\mu^*}(1) > 0$, otherwise we work with with $-f_\mu$. For $V = \nabla^S\mathcal{L}^n_{\mu,\mu^*}(\theta_0, (x_i)_{i=1}^n)$, since $n \geq C^2L^2d^2\ln(d)$, we get from Lemma \ref{lem:approxgrad}
    $$\pr\left(\|V- \nabla^S\mathcal{L}_{\mu,\mu^*}^p(\theta_0)\| \geq \frac{1}{\sqrt{\ln(d)}} \right) \leq \frac{1}{d}.$$
    Note as well that since $\theta_0$ is uniformly distributed in $\bS^d$, then with overwhelming probability $|\langle \theta_0, w^*\rangle| \leq \frac{1}{d^{\frac{1}{4}}}.$ Thus, conditional on $\langle \theta_0, w^*\rangle > 0$, Lemma \ref{lem:weaklearning} implies that the event
   \begin{equation} \label{eq:firstmbound}
       m(\theta_1) \geq \frac{2}{\sqrt{90}}\hat f_{\mu}(1)\hat f_{\mu^*}(1)\sqrt{\frac{\hat f_{\mu}(1)\hat f_{\mu^*}(1)}{90L^6}} -\frac{102L^4}{\sqrt{\log(d)}},
   \end{equation} happens with probability at least $1-\frac{2}{d}$, as long as $\eta$ is small enough.
    Now, by Theorem \ref{thm:quantitativeshift}, we can choose $c > 0$, so that the the event $\frac{2}{\sqrt{90}}\hat f_{\mu}(1)\hat f_{\mu^*}(1)\sqrt{\frac{\hat f_{\mu}(1)\hat f_{\mu^*}(1)}{90L^6}} > c$ happens with probability $1-\frac{\lambda}{2}$. Combining this with \eqref{eq:firstmbound}, for large enough $d$, we see
    $$\pr\left(m(\theta_1) > \frac{c}{2}\right) > 1 - \lambda.$$
    Under this event, Step $2$ amounts to running spherical SGD from a warm start. According to \cite[Theorem 3.2]{ba2022high} (see also Theorem 1.5 from the paper and the discussion in Section 3.2), with our choice of $n$ and learning step we have,
    $$m(\theta_n) \xrightarrow{d\to \infty} 1,$$
    in probability. Finally, we note that the required event $\langle \theta_0, w^*\rangle >0$ happens with probability $\frac{1}{2}.$
\end{proof}
\subsection{Semi-parametric setting}
\label{app:semiparametric_setting}
Recall that in the semi-parametric setting, there is an \emph{unknown} target function $f:\bR \to \bR$ and an unknown signal $w^* \in \bS^d$. We have access to samples $f(\langle {w^*}, x\rangle)$, and the goal is to learn a representation for $f$ and to find the direction $w^*$.

In our representation of $f$ we employ a two-layer neural network with ReLU activation
\begin{equation} \label{eq:twolayers}
    N_{c,\tau,s,\theta}(x) = \frac{1}{\sqrt{K}}\sum\limits_{i=1}^K c_i\sigma(s_i\langle \theta, x\rangle + \tau_i)
\end{equation}
Above $s = \{s_i\}_{i=1}^K \in \{-1,1\}^K$, $c = \{c_i\}_{i=1}^K \in \bR^K$,  $\tau = \{\tau_i\}_{i=1}^K \in \bR^K$, and $\theta \in \bR^d$, are the learnable parameters of the model. Moreover, $K$ is called the width of the network, and $\sigma(t) = \max(t,0)$ is the ReLU activation. Crucially note that the different units, or neurons, only differ in their biases and possible signs, and have the same direction vectors.
\paragraph{Learning functions with small Hermite exponent.} The paper \cite{bietti2022learning} studied the sample complexity of a gradient-based algorithm for the above model. We now explain the guarantees of their algorithm. For the sake of discussion, and because it will be most relevant for us, let us suppose that $f$ has information-exponent $\mathcal{I}(f)=1$.

First, let $(x_i)_{i=1}^n$ be an \emph{i.i.d.} standard Gaussian sample, with corresponding labels $y_i = f(\langle{w^*},x_i\rangle) +\zeta_i$. For $\beta > 0$, consider the regularized $\ell_2$ empirical loss.
$$\mathcal{L}_n(c,\theta, (x_i)_{i=1}^n) = \frac{1}{n}\sum\limits_{i=1}^n\left(N_{c,\tau,s,\theta}(x_i)-y_i\right)^2 + \beta\|c\|_2.$$
During training, we will keep the biases $\tau$ and the signs $s$ fixed, so we suppress the dependence of the loss on these parameters.

The training dynamics of the tunable parameters follow the gradient flow of the empirical loss. For a given assignment of signs $s$ and biases $b$, choose an initialization $c_0$ for the weights and $\theta_0$ for the direction. $c_t$ and $\theta_t$ evolve according to the following dynamics, for some parameter $T' > 0$ to be chosen later,
\begin{align} \label{eq:dynamics}
\frac{dc_t}{dt} &= - {\bf 1}_{t\geq T'}\nabla_c\mathcal{L}_n(c,\theta, (x_i)_{i=1}^n)\nonumber\\
\frac{d\theta_t}{dt} &= - \nabla^{S}\mathcal{L}_n(c,\theta, (x_i)_{i=1}^n)
\end{align}
$\nabla^{S}$ stands for the spherical gradient, with respect to $\theta$. So, as long as $\|\theta_0\|=1$, then for every $t > 0$, $\|\theta_t\|=1$.
We can see that training happens in two stages. At stage only the direction $\theta_t$ changes, then after time $T'$, the weights $c$ start updating as well.

With the correct initialization, we can guarantee to find a good approximation for $w^*$ with high probability.
\vspace{0.5em}

\begin{assumption}[Initialization] \label{ass:init}
The parameters are initialized as follows: $s \sim \mathrm{Uniform}\left(\{-1,1\}^K\right)$, $\tau \sim N(0,\mathrm{I}_K)$, and $\theta_0 \sim \mathrm{Uniform}\left(\mathbb{S}^{d-1}\right)$. For the weight vector $c_0$, we choose uniformly at random $K_0$ coordinates, for some constant $K_0 > 0$, and let $H$ be the subspace spanned by these coordinates. For another parameter $\rho$ and choose $c_0\sim \mathrm{Uniform}\left(\rho\mathbb{S}^{d-1} \cap H\right)$. In other words, $c_0$ is a sparse vector with a fixed norm.
\end{assumption}
Let us now state the main guarantee, which is a special case of \cite[Theorem 6.1]{bietti2022learning} specialized to the case $\mathcal{I}(f) = 1.$
\vspace{0.5em}
\begin{theorem} \label{thm:nonpara}
	Suppose that $\mathcal{I}(f) = 1$, and that $(\theta_0,c_0,b,s)$ are initialized as in Assumption \ref{ass:init}.  For every $\lambda \in (0,\frac{1}{2})$, there is a choice of $\beta, K_0,\rho, T'$ and $K = O(\sqrt{\frac{n}{d^2}})$, such that if $n = \Omega(d^2\log(d))$, then, for $T = \tilde{O}\left(\frac{n}{d}\right)$,
	$$\pr\left(1-\langle w^*, \theta_T\rangle \leq \poly\left(\frac{d^2}{n}\right)\right) \geq \frac{1}{2}-\lambda.$$
	Moreover, if $n = \Omega(d^3)$, we get in addition
	$$\E_{x\sim \cN(0,1)}\left[(N_{\theta_T,c_T,b,s}(x) - f(\langle {w^*},x\rangle))^2\right] \leq \frac{1}{d^\beta},$$
	for some $\beta > 0$.
	
\end{theorem}
\paragraph{Higher information-exponent}
We now prove that even if $f$ has a large information-exponent we can still obtain the same conclusion. We shall use the same algorithm, but change the input distribution and slightly change the dynamics. These modifications lead to Algorithm \ref{alg:semiparam} and the proof of Theorem \ref{thm:semiparameteric}. 
\vspace{0.5em}

\begin{algorithm}[H]
\caption{A solution to the semi-parametric problem with shallow neural network}
\label{alg:semiparam}
\setcounter{AlgoLine}{0}
\LinesNumbered
\KwIn{Algorithm parameters $(\lambda, \beta, K_0, K, \rho, T', T)$ as in Theorem \ref{thm:nonpara},

\quad \quad \quad Initialization parameters $(s,\tilde{\tau},\theta_0,c_0)$ according to Assumption \ref{ass:init},

\quad \quad \quad Gaussian sample $\{\tilde{x}_i\}_{i=1}^n$, shift $\alpha \in \mathbb{R}^d$.}

\BlankLine
Set $x_i = \tilde{x}_i + \alpha$ for $i=1,\dots,n$.

Get noisy labels $y_i = f(\langle w^*, x_i\rangle) + \zeta_i$ for $i=1,\dots,n$.

Set $\tau_0 = \tilde{\tau} - s\langle \theta_0,\alpha\rangle$.

Construct a two-layered neural network $N_{c_0,\tau_0,s,\theta_0}$ as in \eqref{eq:twolayers}.

\BlankLine
Run the following dynamics until time $T$: 

\quad \quad \quad $\frac{dc_t}{dt} = - {\bf 1}_{t\geq T'}\nabla_c(c,\theta, \{x_i\}_{i=1}^n)$.

\quad \quad \quad $\frac{d\theta_t}{dt} = - \nabla^{S}\mathcal{L}_n(c,\theta, \{x_i\}_{i=1}^n)$.

\quad \quad \quad $\frac{d\tau_t}{dt} = -s\langle\alpha,\frac{d\theta_t}{dt}\rangle.$

\Return $\theta_T$
\end{algorithm}

\begin{proof}[Proof of Theorem \ref{thm:semiparameteric}]
In our algorithm, since for every $i \in [n]$, $\tilde{x}_i \sim \mathcal{N}(0,\mathrm{I}_d)$, we get that $x_i \sim \cN(\alpha,\mathrm{I}_d)$. Note that for 
$\mu = \langle {w^*},\alpha\rangle$, and
$f_\mu(x) = f(x+\mu)$, we have the identity $y_i = f(\langle{w^*},x_i\rangle) +\zeta_i= f_\mu(\langle {w^*},\tilde x_i\rangle) +\zeta_i$. By Theorem \ref{thm:mainshift} we know that $\mathcal{I}(f_\mu) = 1$ almost surely and that there exists some constant $m >0$, depending only on the parameters of $f$, such that $\pr\left(|\hat{f}_{\mu}(1)|>m\right)> 1-\lambda$. We continue the proof conditional on this event.

In light of the small-ball estimate, in principle, Theorem \ref{thm:nonpara} applies to $f_\mu$. However, the training dynamics do not exactly match, since the bias $\tau$ now evolves over time. Indeed, since 
$$s_i\langle\theta_t,x_i\rangle + \tau_i = s_i\langle\theta_t,\tilde x_i\rangle + s_i\langle\theta_t,\alpha\rangle+\tau_i,$$ we get that $\tau_{t} = \tau_{0} + s\langle \theta_t, \alpha\rangle$. Moreover, $\tau_0$ now follows a shifted distribution at initialization. For that reason, we instead initalized $\tau_0 \sim \cN(-s\langle\theta_0,\alpha\rangle,\mathrm{I}_K)$, so that, at initialization $\tau_0 + s\langle\theta_0,\alpha\rangle \sim \cN(0,\mathrm{I}_K)$, inline with the initialization in Assumption \ref{ass:init}.

For the dynamics, in addition to the gradient flow in \eqref{eq:dynamics}, we make $\tau_t$ follow. 
$$\frac{d\tau_t}{dt} = -s\langle\alpha,\frac{d\theta_t}{dt}\rangle.$$
With this choice we now have, for every $t \geq 0$ and $i \in [K]$,
$$s_i\langle \theta_t,x_i\rangle + \tau_{i,t} = s_i\langle \theta_t,\tilde x_i\rangle + s_i\langle \theta_t,\alpha\rangle+\tau_{i,t} = s_i\langle\theta_t,\tilde x_i\rangle +\tau_{i,0} ,$$
since $\frac{d}{dt}(s_i\langle\theta_t,\alpha\rangle+\tau_{i,t}) = 0$. We thus get that $\theta_t$ evolves exactly as in Theorem \ref{thm:nonpara} for the function $f_\mu$. The result follows.
\end{proof}

\section{Proofs for Section~\ref{sec:smallball}} \label{sec:appendixsmallball}
\subsection{Weak estimates} \label{sec:weaksmall}
Here we shall prove a weak small-ball estimate which applies to one particular value of $\lambda = 0.1\frac{\sqrt{1000}}{\varepsilon\delta}$, where $\varepsilon$ and $\delta$ are part of Assumption \ref{ass:nicef}. 
\vspace{0.5em}
\begin{proposition}\label{thm:mainshift}
	Let $f:\bR\to \bR$ satisfy Assumption \ref{ass:nicef}, and let $\mu \sim \mathcal{N}(0,1).$ Then,
	$$\pr\left(F_1(\mu) \geq 0.1\frac{\sqrt{1000}}{\varepsilon\delta}\right) \geq 1- \Phi\left(\frac{8000L^2}{\varepsilon^2\delta^2}\right),$$
    where $\Phi$ is the CDF of the Gaussian distribution.
\end{proposition}
Our proof of Proposition \ref{thm:mainshift} has two main components.
\begin{itemize}
	\item We first show that the non-linearity assumption implies that 
	$\mathrm{Var}(F_1(\mu))$ is non-negligible.
	\item We then use the regularity assumption to establish a small-ball estimate for $F_1$. Roughly, we show that $F_1$ is Lipschitz and that Lipschitz functions of a standard Gaussian cannot concentrate too much around $0$, at least not if their variance is large.
\end{itemize}
Before delving into the proof, we state one useful corollary of Proposition \ref{thm:mainshift}. The corollary follows by comparing the Gaussian density to the Lebesgue measure on bounded length intervals and we omit the proof.
\vspace{0.5em}
\begin{corollary} \label{cor:existence}
    Let $f:\bR\to \bR$ satisfy Assumption \ref{ass:nicef}. Then there exists a constant $\tilde c:=\tilde c(\varepsilon,\delta,L)$ such that the set $$A:=\left\{\mu \in \left[-\frac{1}{\tilde c},\frac{1}{\tilde c}\right]:|F_1(\mu)| \geq \tilde c\right\},$$ has Lebesgue measure $|A| \geq \tilde c$. In particular, there exits a point $\mu_0 \in  [-\frac{1}{\tilde c},\frac{1}{\tilde c}]$ such that $F_1(\mu') > \tilde c$.
\end{corollary}
\paragraph{Bounding the variance.} We begin by showing that $\mathrm{Var}(F_1(\mu))$ is large. This will follow from a direct application of the mean value theorem.
\vspace{0.5em}
\begin{lemma} \label{lem:variance}
	Let $f:\bR \to \bR$ and suppose that $f$ satisfies the non-linearity assumption from Assumption \ref{ass:nicef}. Then, for $\mu \sim \mathcal{N}(0,1)$,
	$$\mathrm{Var}(F_1(\mu)) \geq \frac{\varepsilon^2\delta^2}{1000}.$$
\end{lemma}
\begin{proof}
	First by integration by parts in Gaussian space, also called Stein's lemma,
	\begin{equation}\label{eq:steinlemma}
    F_1(\mu) = \E_{x\sim \mathcal{N}(0,1)}[f(x+\mu)x] = \E_{x\sim \mathcal{N}(0,1)}[f'(x+\mu)].
    \end{equation}
	Now, if $\mu'$ is an independent copy of $\mu$, we have,
\begin{align*}
\Var(F_1(\mu)) = \frac{1}{2}\E_{\mu,\mu'}\left[\left(\E_x\left[f'(x+\mu) - f'(x+\mu')\right] \right)^2\right] = \frac{1}{2}\E_{\mu,\mu'}\left[\left(\E_x\left[\int\limits_{x+\mu}^{x+\mu'}f''(t)dt\right]\right)^2\right].
\end{align*}
To finish the proof, for the parameters $c,\delta$ from Assumption \ref{ass:nicef}, define the event,
$$E =\left\{x + \mu \in \left[c-\delta, c-\frac{\delta}{2}\right] \text{ and } x + \mu' \in \left[c+\frac{\delta}{2}, c+\delta\right]\right\}.$$ 
We can now apply the non-linearity assumption and obtain,
\begin{align*}
\frac{1}{2}\E_{\mu,\mu'}\left[\left(\E_x\left[\int\limits_{x+\mu}^{x+\mu'}f''(t)dt\right]\right)^2\right] &\geq\frac{1}{2}\E_{\mu,\mu'}\left[\left(\E_x\left[\int\limits_{x+\mu}^{x+\mu'}f''(t)dt\right]\right)^2\mathbbm{1}_E\right] \\
&\geq \pr\left(E\right)\frac{\varepsilon^2}{2} \geq \frac{\varepsilon^2\delta^2}{800}.
\end{align*}
In the last estimate, we have used the fact that inside $[-1,1]^2$ there is a lower bound on the density of the bivariate normal vector $(x+\mu, x+\mu')$ which is bounded from below by $\frac{1}{100}$. Hence,
$$\pr(E) \geq \frac{1}{100} \mathrm{area}\left(\left[c-\delta, c-\frac{\delta}{2}\right] \times \left[c+\frac{\delta}{2},c+\delta \right]\right) = \frac{1}{100}\frac{\delta^2}{4}.$$
\end{proof}

\paragraph{Small-ball estimate.} To show that $F_1(\mu)$ is anti-concentrated around $0$, we first observe that $F_1(\mu)$ inherits regularity properties from $f$.
\vspace{0.5em}
\begin{lemma} \label{lem:lipschitzshift}
	Let $f:\bR \to \bR$ satisfy the regularity assumption from Assumption \ref{ass:nicef}. That is, $f$ is $L$-Lipschitz. Then, $F_1$ is $\sqrt{\frac{2}{\pi}}L$-Lipschitz.
\end{lemma}
\begin{proof}
	Recall that $F_1(\mu) = \E_{x\sim \cN(0, 1)} \left[f(x+\mu) x \right]$. Thus, using the fact that $f$ is $L$-Lipschitz, for $\mu_1,\mu_2 \in \bR$,
	\begin{align*}
	\frac{|F_1(\mu_1)-F_2(\mu_2)|}{|\mu_1-\mu_2|} &\leq \E_{x\sim \cN(0, 1)} \left[\frac{|f(x+\mu_1) - f(x+\mu_2)|}{|\mu_1 - \mu_2|}|x| \right]
	\leq L\E_{x\sim \cN(0, 1)} \left[|x| \right] = \sqrt{\frac{2}{\pi}}L.
	\end{align*}
\end{proof}
We now show that Lipschitz functions of the standard Gaussian are appropriately anti-concentrated.
\vspace{0.5em}
\begin{lemma} \label{lem:smallball}
	Suppose $g:\bR^d \to \bR$ is $L$-Lipschitz and that $\Var_{x\sim \mathcal{N}(0,1)}(g(x)) \geq 1$. Then, 
	$$\pr_{x\sim\mathcal{N}(0,1)}\left(|g(x)| \leq 0.1\right) < \Phi(8\max(L^2,1)),$$
	where $\Phi$ is the Gaussian cumulative distribution function (CDF).
\end{lemma}
\begin{proof}
	Assume to the contrary, that $\pr_{x\sim\mathcal{N}(0,1)}\left(|g(x)| \leq 0.1\right) = \xi$, and that $\Phi^{-1}(\xi) > 8L^2$, where $\Phi$ is the Gaussian CDF. We will show that this forces the contradiction $\Var_{x\sim \mathcal{N}(0,1)}(g(x)) < 1$.
	Indeed, write $A_0 = \{|g(x)| \leq 0.1\}$ and for $r > 0$, let $A_r$ be its $r$-neighborhood. Note that, since $g$ is $L$-Lipschitz,
	$A_{\frac{r}{L}} \subset \{|g(x)| \leq 0.1 + r\}$. By the Gaussian isoperimetric inequality \cite[Theorem 1.2]{ledoux1993inegalite} we now have,
	$$\Phi(\Phi^{-1}(\xi)+\frac{r}{L}) \leq \pr_{\mathcal{N}(0,1)}\left(A_{\frac{r}{L}}\right)\leq  \pr_{x \sim\mathcal{N}(0,1)}\left(|g(x)| \leq 0.1 + r\right).$$
	It follows that
	\begin{align*}
	\E_{\mathcal{N}(0,1)}[g^2] &= 2\int\limits_0^\infty r\pr_{\mathcal{N}(0,1)}\left(|g(x)| \geq r\right)dr \leq 2\cdot 0.1 + 2\int\limits_0^\infty (0.1 + r)\pr_{\mathcal{N}(0,1)}\left(|g(x)| \geq 0.1 + r\right)dr\\
	&\leq 0.2 + 2\int\limits_0^\infty (0.1 + r)\left(1-\Phi(\Phi^{-1}(\xi)+\frac{r}{L})\right)dr\\
	&\leq 0.2 + \frac{2}{\sqrt{2\pi}}\int\limits_0^\infty\frac{0.1 + r}{\Phi^{-1}(\xi) + \frac{r}{L}}e^{-r^2/2}dr\\
	& \leq 0.2 + \frac{2}{\sqrt{2\pi}}\int\limits_0^\infty\frac{0.1 + r}{8L^2  + \frac{r}{L}}e^{-r^2/2}dr\\
	& \leq  0.2 + \frac{2}{\sqrt{2\pi}}\int\limits_0^{L}\frac{0.1 + r}{8L^2  + \frac{r}{L}}dr + L\frac{2}{\sqrt{2\pi}}\int\limits_{L}^\infty e^{-r^2/2}dr\\
	& \leq  0.2 + \frac{2}{\sqrt{2\pi}}\int\limits_0^{L}\frac{L+0.1}{8L^2}dr + \frac{2}{\sqrt{2\pi}}\int\limits_{L}^\infty re^{-r^2/2}dr \\
	&\leq 0.2 + \frac{1}{\sqrt{8\pi}} + \frac{2}{\sqrt{2\pi}}e^{-L^2/2} < 1,
	\end{align*}
	where the last inequality holds as long $L > 1$.
\end{proof}
	\paragraph{Combining the estimates:} We are now ready to prove Proposition \ref{thm:mainshift}.
	\begin{proof}[Proof of Proposition \ref{thm:mainshift}]
		First, by Lemma \ref{lem:lipschitzshift}, $F_1$ is $L$-Lipschitz, and by Lemma \ref{lem:variance}, we have
		$$\mathrm{Var}\left(\frac{\sqrt{1000}}{\varepsilon\delta}F_1(\mu)\right) \geq 1.$$
		Thus, by Lemma \ref{lem:smallball},
		$$\pr\left(|F_1(\mu)| \leq 0.1\right) = \pr\left(\frac{\sqrt{1000}}{\varepsilon\delta}|F_1(\mu)| \leq 0.1\frac{\sqrt{1000}}{\varepsilon\delta}\right) \leq \Phi\left(\frac{8000L^2}{\varepsilon^2\delta^2}\right).$$
	\end{proof}
\subsection{General estimates} \label{sec:generalsmallball}
To boost Proposition \ref{thm:mainshift} and obtain estimates for general $\lambda$, we begin by showing that $F_1(\mu)$ can be extended to an entire analytic function.
\vspace{0.5em}
\begin{lemma} \label{lem:analytic}
	Suppose that $f$ is $L$-Lipschitz. Then, there exists an analytic function $\Psi:\mathbb{C} \to \mathbb{R}$, such that $\Psi(\mu) = F_1(\mu)$ for every $\mu \in \bR$. Furthermore $\Psi$ satisfies $|\Psi(z)| \leq \sqrt{(|z|+1)}e^{|z|}$ for every $z \in \mathbb{C}.$
\end{lemma}
\begin{proof}
	Note that the expression $F_1(\mu) = \E_{x\sim \mathcal{N}(0,1)}[f'(x+\mu)]$, as in \eqref{eq:steinlemma}, implies that $F_1(\mu) = u(1,\mu)$ where $u$ is a solution to the heat equation
	$$\partial_t u(t,\mu) = \frac{1}{2}\partial_{xx}u(t,\mu),\ \ \ u(0,\mu) = f'(\mu).$$
	By standard results on solutions to the heat equations \cite[Section 2.3.3.c]{evans1998partial}, $F_1(\mu)$ is real analytic. Thus, for $|\mu|$ small enough
	$$F_1(\mu) = \sum\limits_{k=0}^\infty \frac{F_1^{(k)}(0)}{k!}\mu^k.$$
	We turn to calculate the derivatives,
	$$F_1^{(k)}(0) = \frac{\partial^k}{\partial\mu^k}\E_{x\sim \mathcal{N}(0,1)}[f'(x+\mu)]\Big\vert_{\mu = 0} = \E_{x\sim \mathcal{N}(0,1)}[f^{(k+1)}(x)] = \E_{x\sim \mathcal{N}(0,1)}[f(x)H_{k+1}(x)]=\hat{f}(k+1) ,$$
	where the second identity follows from successive integration by parts in Gaussian space.
	So,
	$$F_1(\mu) = \sum\limits_{k=0}^\infty \frac{\hat{f}(k+1)}{k!}\mu^k.$$
	We now show that the above power series has an infinite radius of convergence.
	Indeed, since $\mathrm{Var}_{x\sim \mathcal{N}(0,1)}(f(x)) = 1$ we know that
	$$\sum\limits_{k=0}^\infty \hat{f}^2(k+1)= 1.$$
	The Cauchy-Schwartz inequality now implies, for every $z \in \mathbb{C}$,
	$$\sum\limits_{k=0}^\infty \frac{|\hat{f}(k+1)|}{k!}|z|^k \leq \sqrt{\sum\limits_{k=0}^\infty \frac{\hat{f}^2(k+1)}{(k+1)!} \sum\limits_{k=0}^\infty \frac{k+1}{k!}|z|^{2k} }  \leq \sqrt{ \sum\limits_{k=0}^\infty \frac{k+1}{k!}|z|^{2k} } = \sqrt{(|z|+1)}e^{|z|}.$$
	We now define the analytic function,
	$$\Psi(z) = \sum\limits_{k=0}^\infty \frac{\hat{f}(k+1)}{k!}z^k.$$
	By the uniqueness theorem $\Psi$ must agree with $F_1$ on the real line and we have already established $\Psi(z) \leq \sqrt{(|z|+1)}e^{|z|}.$ 
\end{proof}
We shall now use the analyticity of $F_1(\mu)$ to derive small-ball estimates. For these estimates, we shall utilize the local bounds of Nazrov, M. Sodin, and Volberg \cite[Theorem A]{nazarov2003local}. Below we use $\mathbb{D}$ to denote the unit ball in $\mathbb{C}$.
\vspace{0.5em}
\begin{theorem} \label{thm:NSV}
	Let $F:\mathbb{D} \to \mathbb{C}$ be analytic and satisfy $\sup\limits_{z \in \mathbb{D}} |F(z)| \leq 1$ and $|F(0)| > 0$. Set, $M(F)$ to be the unique number satisfying
	$$\left|\left\{\mu \in \left[-\frac{3}{4},\frac{3}{4}\right] : |F(\mu)| \geq M(F)\right\}\right| = \frac{2}{3e}.$$
	Then, if $\sigma = -\frac{3}{4}\log(|F(0)|)$, for every $\lambda > 1$,
	$$\left|\left\{\mu \in \left[-\frac{3}{4},\frac{3}{4}\right] : |F(\mu)| \leq \left(C\lambda\right)^{\sigma}M(F)\right\}\right| \leq \frac{1}{\lambda},$$
	where $C > 0$ is some universal constant.
\end{theorem}
To apply Theorem \ref{thm:NSV} to our analytic function $F_1$ we shall need to localize it. To localize $F_1$ we use the constant $\tilde c = \tilde c(\varepsilon,\delta, L)$ and the point $\mu'$ promised by Corollary \ref{cor:existence}. For a large radius $R > 1$, define the localized function by
$$F_1^R(z) = \frac{1}{e^{2(R+\frac{1}{\tilde c})}}F_1\left(Rz + \mu'\right).$$
Let us summarize the properties of $F_1^R$.
\begin{lemma} \label{lem:properties}
	Let $R > 1$, then,
	\begin{enumerate}
		\item $|F_1^R(0)| \geq e^{-2R}e^{-\frac{2}{\tilde c}}\tilde c$.
		\item For every $z \in \mathbb{D}$, $|F_1^R(z)| \leq 1.$
		\item Volume comparison: for every $\eta >0$
		$$\left|\left\{\mu\in \left[\mu'-R,\mu'+ R\right] : |F_1(\mu)| \leq \eta \right\}\right| \leq R\left|\left\{\mu\in \left[-\frac{3}{4},\frac{3}{4}\right] : |F_1^R(\mu)| \leq \frac{1}{e^{2(R+\frac{1}{\tilde c})}}\eta \right\}\right|.$$
		\item Quantile bound: Let $M_R$ be the unique number satisfying
		$$\left|\left\{\mu \in \left[-\frac{3}{4},\frac{3}{4}\right] : |F_1^R(\mu)| \geq M(F)\right\}\right| = \frac{2}{3e}.$$ Then,
		$$M_R \geq \frac{\tilde c^2}{e^{2(R+\frac{1}{\tilde c})}R}.$$
	\end{enumerate}
\end{lemma}
\begin{proof}
	The first property follows by observing $F^R(0) = e^{-2\left(R+\frac{1}{\tilde c}\right)} F_1(\mu')$ and the defining property of $\mu'$ from Corollary \ref{cor:existence}.
	
	For the second property, note that if $|z| \leq 1$, then
	$$|F^R(z)| \leq e^{-2\left(R+\frac{1}{\tilde c}\right)} \sqrt{R +|\mu'| + 1}e^{|R| + |\mu'|} \leq 1,$$
	which follows again by the properties of $\mu'$ in Corollary \ref{cor:existence} as well as the growth condition of $F_1$ from Lemma \ref{lem:analytic}.
	
	The third property follows from the  change of variables $z \to Rz + \mu'$, mapping $\left[-\frac{3}{4},\frac{3}{4}\right]$ to the interior of $\left[\mu'-R, \mu' + R\right].$
	
	For the final property, we shall use Markov's inequality. Let $X \sim \mathrm{Uniform}(\left[-\frac{3}{4},\frac{3}{4}\right])$, so that for any $\eta > 0$,
	$$
	\frac{3}{2}\left|\left\{\mu \in \left[-\frac{3}{4},\frac{3}{4}\right] : |F_1^R(\mu)| \geq \eta \right\}\right| = 
	\pr\left(|F_1^R(X)| \geq \eta\right).$$
	We have already seen that $|F_1^R(X)| \leq 1$ almost surely, and so if 
    $\eta = e\E\left[|F_1^R(X)|\right]$, 
	$$	\pr\left(|F_1^R(X)| \geq \eta\right) \leq \frac{1}{e}.$$
	Combined with the previous identity this shows
	$$M_R \geq e\E\left[|F_1^R(X)|\right].$$
	Let $A$ be the set defined in Corollary \ref{cor:existence} and let $\tilde{A} = \frac{1}{R}A - \frac{\mu'}{R}$. By definition of $A$,
	$$\E\left[|F_1^R(X)|\right] \geq \frac{\tilde c}{e^{2(R+\frac{1}{\tilde  c})}}\Pr\left(X \in \tilde{A}\right) \geq \frac{\tilde  c^2}{e^{2(R+\frac{1}{\tilde  c})}R},$$
	where we used the volume bound form Corollary \ref{cor:existence} and applied the transformation $z \to \frac{z-\mu'}{R}.$
\end{proof}
We shall now use the above properties combined with Theorem \ref{thm:NSV} to establish small-ball estimates for $F_1(\mu)$ and prove Theorem \ref{thm:quantitativeshift}.

\vspace{1em}
\begin{proof}[Proof of Theorem \ref{thm:quantitativeshift}]
	Fix some small $\xi > 0$. We'll begin by finding $\eta > 0$ satisfying, $\pr\left(|F_1(\mu)| \leq \eta\right) \leq \xi.$
    Towards this task, let $\tilde c$ and $\mu'$ be as in Corollary \ref{cor:existence}, and suppose for the sake of simplicity that $\frac{1}{\tilde 
 c}\leq\sqrt{\log\left(\frac{2}{\xi}\right)}$. Set $R = 4 \sqrt{\log\left(\frac{2}{\xi}\right)}$, so that by the law of total probability, and standard Gaussian concentration bounds, we have
	\begin{align*}
	\pr\left(|F_1(\mu)| \leq \eta\right) &\leq \pr\left(|\mu-\mu'|\geq R\right) + \pr\left(|F_1(\mu)| \leq \eta \text{ and }\ |\mu-\mu'|\leq R\right)\\
	&\leq \pr\left(|\mu|\geq \sqrt{\log\left(\frac{2}{\xi}\right)}\right) + \left|\left\{\mu\in \left[\mu'-R,\mu'+ R\right] : |F_1(\mu)| \leq \eta \right\}\right|\\	
	&\leq \frac{\xi}{2} + \left|\left\{\mu\in \left[\mu'-R,\mu'+ R\right] : |F_1(\mu)| \leq \eta \right\}\right|.
	\end{align*}
	Thus, it will be enough to find an $\eta$ which satisfies:
	$$\left|\left\{\mu\in \left[\mu'-R,\mu'+ R\right] : |F_1(\mu)| \leq \eta \right\}\right| \leq \frac{\xi}{2}.$$
	Let now $F_1^R$ be defined as in Lemma \ref{lem:properties}. By the third property in the lemma, the required bound will follow, if we can establish
		$$\left|\left\{\mu\in \left[-\frac{3}{4},\frac{3}{4}\right] : |F_1^R(\mu)| \leq e^{-(2R + \frac{1}{\tilde c})}\eta \right\}\right| \leq \frac{\xi}{2R}.$$
	By Lemma \ref{lem:properties} we know that for any $z \in \mathbb{C}$ with $|z| \leq 1$, $|F_1^R(z)|\leq 1$. Hence, we invoke Theorem \ref{thm:NSV} which implies,
	\begin{align*}
	\left|\left\{\mu\in \left[-\frac{3}{4},\frac{3}{4}\right] : |F_1^R(\mu)| \leq e^{-(2R + \frac{1}{\tilde c})}\eta \right\}\right| &= \left|\left\{\mu\in \left[-\frac{3}{4},\frac{3}{4}\right] : |F_1^R(\mu)| \leq \left(\frac{e^{-(2R + \frac{1}{\tilde c})}\eta}{M_R}\right)^{\frac{\sigma}{\sigma}} M_R\right\}\right|\\
	&\leq \left(\frac{C\eta}{M_Re^{(2R + \frac{1}{\tilde c})}}\right)^\frac{1}{\sigma},
	\end{align*}
	where $C > 0$ is an absolute constant and, by Lemma \ref{lem:properties}, $\sigma = -\frac{3}{4}\log(|F_1^R(0)|) \leq 2R  + \frac{2}{\tilde c} + \log(\frac{1}{\tilde c})\leq 5R$.
	Choose now $\eta = \frac{1}{C}\left(\frac{\xi}{2R}\right)^{\sigma}M_Re^{2R + \frac{1}{\tilde c}}$ for which we obtain
	$\pr\left(|F_1(\mu)| \leq \eta\right) \leq \xi$.
	It remains to bound $\eta$ from below. First, by Lemma \ref{lem:properties}, $M_re^{2R + \frac{1}{\tilde c}} \geq \frac{\tilde c^2}{R}$, and by the choice of $R$ we get
	$$\eta > \frac{\tilde c^2}{CR}\left(\frac{\xi}{2R}\right)^{\sigma} = \frac{\tilde c^2}{4C\sqrt{\log\left(\frac{2}{\xi}\right)}}\left(\frac{\xi}{8\sqrt{\log\left(\frac{2}{\xi}\right)}}\right)^{\sigma} \geq \frac{\tilde c^2}{4C} \xi^{3\sigma} \geq \frac{\tilde c^2}{4C} \xi^{60\sqrt{\log\left(\frac{2}{\xi}\right)}},$$
	where we have used the bound on $\sigma$ and the elementary inequality 
	$\frac{1}{\sqrt{\log\left(\frac{2}{\xi}\right)}} \geq \frac{1}{\sqrt{\xi}}.$
	We have thus established 
	$$\pr\left(|F_1(\mu)| \leq \frac{\tilde c^2}{4C} \xi^{60\sqrt{\log\left(\frac{2}{\xi}\right)}}\right) \leq \xi.$$
	Set $\lambda = \frac{\tilde c^2}{4C} \xi^{60\sqrt{\log\left(\frac{2}{\xi}\right)}}$, so that, for an appropriate constant $c>0$, which depends only on $\tilde c$ and $C$, $\xi\leq \exp\left(-c\log\left(\frac{1}{\lambda}\right)^{\frac{2}{3}}\right),$ and
	$$\pr\left(|F_1(\mu)| \leq \lambda\right) \leq \exp\left(-c\log\left(\frac{1}{\lambda}\right)^{\frac{2}{3}}\right).$$
\end{proof}

\section{Proofs for sparse Boolean functions}
\label{app:junta_proofs}
Let us first state the following lemma, which is a restatement of the Carbery-Wright inequality~(\cite{carbery2001distributional}).
\vspace{0.5em}
\begin{lemma}[\cite{carbery2001distributional}] \label{lem:carbery-wright}
    Let $P: [-\delta, \delta]^k \to \bR$, for $\delta>0$, be a polynomial of degree at most $K=O(1)$, and let $\bmu \in \bR^k$ be a random vector such that $ \bmu \sim \Unif [-\delta,\delta]^{\otimes k}$ . Then, for $\varepsilon>0$,
    \begin{align*}
        \pr_{\bmu}( | P(\bmu) | < \varepsilon \sqrt{ \E [P(\bmu)^2]}  ) \leq O(\varepsilon^{1/K}).
    \end{align*}
\end{lemma}
Thus, the Carbery-Wright inequality says that an estimate of the probability that a random polynomial takes values near zero can be obtained by bounding its second moment.

\subsection{Proof of Proposition~\ref{prop:fourier_firstorder}}
For all $j \in T$, let us write $\hat f_{\bmu}(\{j\})$ in terms of the Fourier-Walsh coefficients under the uniform distribution:
\begin{align*}
   \sqrt{1-\mu_j^2}\cdot  \hat f_{\bmu}(\{ j \}) &= \E_{x \sim \mathcal{D}_{\bmu}} \left[ f(x) (x_j -\mu_j)\right]\\
   & = \sum_{S \subseteq [d]} \hat f(S) \mu_j^{2\mathds{1}(j \not\in S)} \cdot \frac{\chi_{S}(\mu)}{\mu_j} - \mu_j \cdot \sum_{S \subseteq [d]} \hat f(S) \chi_S(\mu)  \\
   & = \sum_{S: j \in S} \hat f(S) \cdot \chi_{S\setminus j}(\mu) \cdot (1-\mu_j^2).
\end{align*}
Recall the definition of the Boolean influence: $\Inf_j(f) = \sum_{S: j \in S} \hat f(S)^2$~(\cite{o'donnell_2014}). Moreover, since the characters $\chi_S$ are orthogonal with respect to $\Unif [-\eta, \eta]^{\otimes k}$,
    \begin{align*}
        \E\left[ (1-\mu_j^2)\cdot  \hat f_{\bmu}(\{ j \})^2 \right] &= \E\left[\left(\sum_{S: j \in S} \hat f(S) \cdot \chi_{S\setminus j}(\mu) \cdot (1-\mu_j^2)\right)^2\right]\\
        &= \E\left[(1-\mu_j^2)^2\right]\sum_{S: j \in S} \hat f^2(S)\E\left[\chi_{S\setminus j}(\mu)^2\right]\\
        &\geq \Inf_j(f) \left(\frac{\eta^2}{3}\right)^{k-1}(1-\frac{2}{3}\eta^2+ \frac{1}{5}\eta^4),
    \end{align*}
    In particular, by Lemma \ref{lem:carbery-wright} we get that, 
    \begin{align*}
\pr \left(|\hat f_{\bmu}( \{ j \}) | < \varepsilon \eta^{k-1} \sqrt{\Inf_j(f)} \right) = O\left(\varepsilon^{\frac{1}{k+1}}\right).
    \end{align*}

\subsection{Proof of Theorem~\ref{thm:juntas_formal}}













\paragraph{Setup and algorithm.}
Without loss of generality, we assume that $T = [k]$, where $T$ denotes the set of relevant coordinates. We further assume that $ |f(x)| \leq R$, for all $x \in \{ \pm 1 \}^d$. 
We train our network with layerwise stochastic gradient descent (SGD), defined as follows:
\begin{align*}
    &w_{ij}^{t+1} = w_{ij}^t - \gamma_t \frac{1}{B} \sum_{s=1}^B\partial_{w_{ij}^t} L(f(x^s),\NN(x^s;\theta^t) ), \\
    & a_i^{t+1} = a_i^t - \xi_t \frac{1}{B} \sum_{s=1}^B\partial_{a_{i}^t} L(f(x^s),\NN(x^s;\theta^t) ),
\end{align*}
where $L$ is a loss function, to be defined soon, $B \in \bN$ is the batch size, and $\gamma_t,\xi_t \in \bR$ are appropriate learning rates.
We set $\gamma_t = \gamma \mathds{1}(t=0)$ and $\xi_t = \xi \mathds{1}(t>0)$, for $\xi ,\gamma \in \bR$, meaning that we train only the first layer for one step, and then only the second layer until convergence. Such $1$-step gradient analysis have been used in previous works~(\cite{daniely2020learning,barak2022hidden,dandi2023two,cornacchia2023mathematical}). We initialize the first layer weights $w_{ij}^0 =0$ for all $i \in [N],j\in [d]$, and the second layer weights $a_i^0 =\kappa $, for $i \in [N]$, where $\kappa>0$ is a constant. We do not train the bias neurons. We initialize the biases as $ b_i^0 =\kappa$. After the first step of training, the biases are drawn from $ b_i^1 \overset{i.i.d.}{\sim} \Unif [-L,L]$, where $L \geq \kappa $. This is necessary to guarantee enough diversity among the hidden neurons, see Lemma~\ref{lem:main_second_phase}.
In the first phase of training, we use the covariance-loss, defined as follows:
\bigskip
\begin{definition}[Covariance loss]\label{def:covariance_loss}
    Let $f : \cX \to \bR$ be a target function and let $\cD$ be an input distribution over the input space. Let $\hat f: \cX \to \bR$ be an estimator. The covariance loss is defined as:
    \begin{align}
        L_{\rm cov}( x, f, \hat f, \cD) = -( f(x) -\E_{x \sim \cD}[f(x)] )\cdot (\hat f(x)-\E_{x \sim \cD}[\hat f(x)]).
    \end{align}
\end{definition}
This choice of loss is particularly convenient because it allows us to get non-zero initial gradients on the relevant coordinates, and zero initial gradients outside the support of the target function, simplifying our construction. We believe, however, that with further technical work, the argument could be extended to other losses that are more popular in practice. In the second phase, we go back to standard convex losses (e.g. the squared loss), to control the variance of our estimator. In the case of binary classification tasks (i.e. $f(x)^2=1$ for all $x\in\{ \pm 1 \}^d$) one could use the covariance loss in the second stage of training as well, since for those tasks a low covariance loss corresponds to large classification accuracy (see e.g.~\cite{abbe2023provable}, Appendix D). We also note that for balanced $f$, i.e. if $\E_{x \sim \cD}[f(x)]=0$, the covariance loss corresponds to the correlation loss, used e.g. in~(\cite{barak2022hidden}).

\paragraph{First layer training.} 
Let us first compute the initial \textit{population gradients} of the first layers' weights, which we denote by $\bar G_{w_{ij}^0}$.
\begin{align*}
    \bar G_{w_{ij}^0} :& = \E \left[ \partial_{w_{ij}^0} L_{\rm cov} (x,f,\NN(x;\theta^0),\cD_{\bmu} )\right] \\
    &= \E[ a_i^0 \mathds{1}(w_i^0 x+b_i^0>0) x_j \cdot f(x)] -\E [ a_i^0 \mathds{1}(w_i^0 x+b_i^0>0) x_j]\cdot \E [f(x)]  \\
   & \overset{(a)}{=} \kappa\left(\E[x_j f(x)] -   \E[x_j]\E [f(x)]\right)\\
    & = \kappa \E[f(x) (x_j-\mu_j)]\\
    & \overset{(b)}{=} \kappa \hat f_{\bmu}( \{ j \}) \sqrt{1-\mu_j^2}=: \alpha_j,
\end{align*}
where $(a)$ holds because of the initialization that we have chosen, and in $(b)$ we used the definition of $\hat f_{\bmu}(\{j\})$.
The following lemma bounds the discrepancy between the effective gradients, estimated through $B$ samples, and the population gradients.
\vspace{0.5em}
\begin{lemma} \label{lem:estimated_grad}
    Let $G_{w_{ij}^0}: = \frac 1B \sum_{s=1}^B \partial_{w_{ij}^0} L_{\rm cov} (x^s,f,\NN(x^s;\theta^0),\cD_{\bmu} ) $ denote the effective gradient. For $\varepsilon>0$, if $B \geq 2 \zeta^{-2} \kappa^2 R^2 \log\left( \frac{Nd}{\varepsilon} \right)$, with probability $1-2\varepsilon$, then
    \begin{align*}
        | G_{w_{ij}^0} - \bar G_{w_{ij}^0} | \leq \zeta, 
    \end{align*}
    for all $i \in [N]$ and for all $j \in [d]$.
\end{lemma}
\begin{proof}
    We apply Hoeffding's inequality, noticing that $|G_{w_{ij}^0}| \leq 2 R \kappa  $,
    \begin{align*}
        \pr \left( | G_{w_{ij}^0} - \bar G_{w_{ij}^0} | > \zeta  \right)  \leq 2\exp \left( - \frac{\zeta^2 B}{2 \kappa^2 R^2} \right) \leq \frac{2\varepsilon }{Nd}.
    \end{align*}
    The result follows by a union bound.
\end{proof}
The following lemma guarantees that there is enough diversity among the hidden neurons.
\vspace{0.5em}
\begin{lemma} \label{lem:main_first_phase}
Let $\alpha_j = \kappa \hat f_\mu( \{ j \}) \sqrt{1-\mu_j^2} $. For $\varepsilon>0$, there exists a constant $C>0$ such that,
with probability $1-O(\varepsilon^{\frac{1}{k+1}})$ over $\bmu$:
    \begin{itemize}
    \item For all $ s,t \in \{ \pm 1 \}^k$, such that $s \neq t$, $ \Big| \sum_{j=1}^k \alpha_j (s_j-t_j)  \Big| \geq C \kappa  \eta^{k+1}\varepsilon $.
\end{itemize}
\end{lemma}

\begin{proof}
Let us consider
\begin{align*}
    P(\bmu) &:=  \sum_{j=1}^k \sum_{S: j \in S} \hat f(S) \cdot \chi_{S \setminus j}(\bmu) \cdot (1-\mu_j^2) c_j,
\end{align*}
which is a polynomial of degree $k+1$ in $\bmu$, and where we denoted $c_j:= s_j-t_j$. In order to apply lemma~\ref{lem:carbery-wright}, let us bound the second moment of $P(\bmu)$.
\begin{align*}
    \E [P(\bmu)^2] &= \sum_{j,l=1}^k c_j c_l \sum_{S,T: j \in S, l \in T} \hat f(S) \hat f(T) \E[ \chi_{S \setminus j}(\bmu) \chi_{T \setminus l} (\bmu) (1-\mu_j^2) (1-\mu_l^2)]\\
    & \overset{(a)}{=}\sum_{j,l=1}^k c_j c_l \E[(1-\mu_j^2)(1-\mu_l^2)] \sum_{S} \hat f(S \cup j) \hat f(S \cup l) \left(\frac{\eta^2}{3}\right)^{|S|}\\
    & = \sum_S \left(\frac{\eta^2}{3}\right)^{|S|} \E \left[ \left( \sum_{j=1}^k c_j (1-\mu_j^2) \hat f(S \cup j) \right)^2 \right] \\
    & \geq  \left(\frac{\eta^2}{3}\right)^{k-1}  \sum_S \E \left[ \left( \sum_{j=1}^k c_j (1-\mu_j^2) \hat f(S \cup j) \right)^2 \right] 
\end{align*}
where in $(a)$ we used the fact that odd moments of a centered uniform distributions are zero.
Since $s \neq t$, there exists at least one $j$ such that $c_j\neq 0$. Since such $j$ is in the support of $f$, for at least one set $S$, the term inside the expectation is a non-zero polynomial, thus its second moment is at least $\Omega(\eta^4)$. It follows that $\E[P(\bmu^2)] =\Omega(\eta^{2(k+1)} )$. Since $ \sum_{j=1}^k \alpha_j (s_j -t_j) = C \kappa P(\bmu) $, for a constant $C>0$, the result follows by Lemma~\ref{lem:carbery-wright} and by union bound.
\end{proof}

\paragraph{Second layer training.}
We show that the previous lemmas imply that there exists an assignment of the second layer's weights that achieves small error. 
\vspace{0.5em}
\begin{lemma} \label{lem:main_second_phase}
    Assume that $b_i \sim \Unif[-L,L]$, with $L\geq \kappa $. Let $\alpha_j = \kappa \hat f_{\bmu}(\{j\})\sqrt{1-\mu_j^2} $, for $j \in [k]$. For $\varepsilon_1, \varepsilon_2>0$, if the number of hidden neurons $N>\Omega(L\log(1/\varepsilon_2) \eta^{-(k+1)}\varepsilon_1^{-1})$, with probability $1-O(\varepsilon_1^{\frac{1}{k+1}} +\epsilon_2)$, there exists a set of hidden neurons $\{i\}_{i \in [2^k]}$ and a vector $a^* \in \bR^{2^k}$ with $\|a^*\|_\infty \leq O(\varepsilon_1^{-1}\eta^{-(k+1)}) $ such that for all $x \in \{ \pm 1 \}^d$,
    \begin{align} \label{eq:35}
       f(x) = \sum_{i=1}^{2^k} a_i^*  \ReLU\left(\gamma  \sum_{j=1}^k \alpha_j x_j + b_i \right).
    \end{align}
\end{lemma}
\begin{proof}
    For all $s \in \{ \pm 1 \}^k$, let $v_s := \gamma \sum_{j =1}^k \alpha_j s_j$, and let us order the $(v_{s_l})_{l \in [2^k]} $ in increasing order, i.e. such that $v_{s_l}<v_{s_{l+1}}$ for all $l \in [2^k-1]$. For simplicity, we denote $ v_l=v_{s_l} $. By Lemma~\ref{lem:main_first_phase}, we have that with probability $1-O(\varepsilon_1^{\frac{1}{k+1}})$ over $\bmu$, $\min_{l \in [2^k-1]} v_{l+1}-v_l > C \gamma \kappa \varepsilon_1 \eta^{k+1}$, for some constant $C>0$. If $N>\Omega(L\log(1/\varepsilon_2) \eta^{-(k+1)}\varepsilon_1^{-1}) $, then with probability $ 1-O(\varepsilon_2)$ there exists a set of $2^k$ hidden neurons $(b_l)_{l \in [2^k]}$ such that for all $l \in [2^k]$, $b_l \in (v_{l-1},v_l)$, where for simplicity we let $v_0 = -L$. Let us define the matrix $M \in \bR^{2^k \times 2^k}$, with entries:
    \begin{align*}
        M_{n,m} = \ReLU(v_n -b_m), \qquad n, m \in [2^k].
    \end{align*}
    Then, by construction $M$ is lower triangular, i.e. $ M_{n,m}=0$ if $m\geq n+1$. Furthermore, by the construction above, the diagonal entries of $M$ are non-zero. Thus, $M$ is invertible. Let us denote by $F \in \bR^{2^k}$ the vector such that for all $l \in [2^k]$, the $l$-th entry is given by $F_{l} = f(s_l) $. Then, $a^* = M^{-1}F$ and
    \begin{align*}
        \|a^*\|_\infty &\leq \| M^{-1}\|_{\infty} \| F \|_{\infty} \\
        &\leq C \cdot \frac{1}{\gamma \kappa \varepsilon_1 \eta^{k+1}}\cdot R,
    \end{align*}
    for a constant $C>0$.
\end{proof}

\noindent 
By combining the lemmas above, we obtain that for $\kappa, \gamma,R,L = \theta(1)$,  $B \geq  \Omega(d \log(d)^2 \log(Nd/\varepsilon))$, $N \geq \Omega(\log(1/\varepsilon) \eta^{-(k+1)} \varepsilon^{-1})$ with probability $1- O(\varepsilon^{c})$, for some $c>0$, there exists a set of $2^k$ hidden neurons $\{ i\}_{i \in [2^k]}$ such that 
\begin{align*}
    &\forall j \in [k], i \in [2^k]: | w_{ij}^1 - \gamma \alpha_j| < \frac{1}{\sqrt{d}\log(d)};\\
    &\forall j \not\in [k], i \in [2^k]: | w_{ij}^1| < \frac{1}{\sqrt{d}\log(d)};
\end{align*}
and the $b_i$ are such that~\eqref{eq:35} holds.
\noindent
By a slight abuse of notation, let us denote by $a^* \in \bR^N$ the $N$-dimensional vector whose entries corresponding to the hidden neurons $\{ i\}_{i \in [2^k]}$ are given by Lemma~\ref{lem:main_second_phase}, and the other entries are zero. For all $i \in [N]$, let $w^*_i \in \bR^d$ be such that $w_{ij}^* = \gamma \alpha_j \mathds{1}(j \in [k])$. Let $\hat \theta = (a^*_i, w_i^1,b_i)_{i \in [N]} $ and $\theta^* = (a^*_i, w^*_i, b_i)_{i \in [N]}$. Then, for all $x \in \{ \pm 1 \}^d$ we have:
\begin{align}
    \left(f(x) - \NN(x;\hat \theta)\right)^2 &\leq  \left(f(x) - \NN(x;\theta^*)\right)^2 +\left(\NN(x;\theta^*)-\NN(x;\hat \theta)\right)^2 \label{eq:43}\\
    & \overset{(a)}{\leq} \left(\sum_{i=1}^N a_i^* \left( \ReLU( w_i^1 x +b_i) - \ReLU(w_i^* x +b_i)  \right)\right)^2 \nonumber
\end{align}
where $(a)$ follows because, by Lemma~\ref{lem:main_second_phase}, the first term of~\eqref{eq:43} is zero. Note that,
\begin{align}
    \Big| \ReLU( w_i^1 x +b_i) - \ReLU(w_i^* x +b_i) \Big| 
    &\leq  \Big| \sum_{j=1}^k (w_{ij}^1 - \gamma \alpha_j) x_j \Big|+ \Big| \sum_{j=k+1}^d w_{ij}^1 x_j \Big| \\
    & \leq \frac{2k}{\sqrt{d}\log(d)} +\Big| \sum_{j=k+1}^d w_{ij}^1 x_j \Big|.
\end{align}
Now, $\sum_{j=k+1}^d w_{ij}^1 x_j = \sum_{j=k+1}^d w_{ij}^1 \mu_j + \sum_{j=k+1}^d w_{ij}^1 (x_j-\mu_j): =m_i+Z_i $. Then, 
\begin{align}
    \E Z_i^2 = \sum_{j=k+1}^d w_{ij}^2 \Var(x_j) =O(1/\log(d)^2 ).
\end{align} 
One the other hand, since $\mu_j$ are i.i.d. in $[-\eta,\eta]$, $m_i$ is sub-Gaussian with $\Var(m_i) \leq \eta^2/\log(d)^2$. Hence, 
\begin{align}
    \pr(|m_i|>t) \leq 2 \exp(-t^2\log^2(d) /2\eta^2).
\end{align}
Setting $t= O(\eta/\sqrt{\log(d)})$, we get that with probability $1-1/d$ over $\mu$, $m_i^2 <\eta^2/\log(d)$. Combining terms,
\begin{align}
    \E_{x\sim \cD_\mu} \left(\sum_{j=k+1}^d w_{ij}^1 x_j \right)^2 = O(\eta^2/\log(d)).
\end{align}


\noindent 
Thus, 
\begin{align}
    \E_{x\sim \cD_\mu} &\left[ \left(\sum_{i=1}^N a_i^* \left( \ReLU( w_i^1 x +b_i) - \ReLU(w_i^* x +b_i)  \right)\right)^2 \right] \\
    &\leq \| a^* \|_2^2  \cdot \sum_{i=1}^N \E_{x \sim \cD_{\mu}}\left[\left( \ReLU( w_i^1 x +b_i) - \ReLU(w_i^* x +b_i)  \right)^2\right] \\
    & \leq N \|a_i^*\|_2^2 \frac{1}{\log(d)^{1/2}} = O\left(\frac{\varepsilon^{-2}\eta^{-2(k+1)}}{\log(d)^{1/2}}\right).\label{eq:45}
\end{align}

\noindent 
For fixed $\varepsilon, \eta,k$ and for $d$ large enough, the right hand side of~\eqref{eq:45} is below $\delta/2$.
To conclude, we use the following well-known result on the convergence of SGD on convex losses, to show that training only the second layer with a convex loss achieves small error.
\vspace{0.5em}
\begin{theorem}[\cite{shalev2014understanding}] \label{thm:convergence_SGD_martingale}
    Let $\cL$ be a convex function and let $a^* \in \argmin_{\|a\|_2 \leq \cB} \cL(a) $, for some $\cB>0$. For all $t$, let $\alpha^{t}$ be such that $\E\left[ \alpha^{t} \mid a^{t} \right] =  -\nabla_{a^{t}} \cL(a^{t}) $ and assume $\| \alpha^{t} \|_2 \leq A $ for some $A>0$. If $a^{(0)} = 0 $ and for all $t \in [T]$ $a^{t+1} = a^{t} +\gamma \alpha^{t} $, with $\gamma = \frac{\cB}{A \sqrt{T}}$, then
    \begin{align*}
        \frac{1}{T} \sum_{t=1}^T \cL(a^{t}) \leq \cL(a^*) + \frac{\cB A}{\sqrt{T}}.
    \end{align*}
\end{theorem}
We choose $\cB = \Omega(\varepsilon^{-1}\eta^{-(k+1)})$. We train with any convex loss, with $\| \alpha^t\|_2 \leq A$, for an appropriate $A$. This can achieved either by gradient clipping, or by computing explicit bound on the gradients, depending on the loss. 
By~\eqref{eq:45}, $\cL(a^*)\leq \delta/2$, for $d$ large enough. Thus, to achieve error at most $\delta$, we need at least $T = \Omega\Big(\frac{A^2}{\delta^2 \varepsilon^2 \eta^{2(k+1)}}\Big)$ training steps.


\end{document}